\documentclass{article}

\usepackage[utf8]{inputenc} 
\usepackage[english]{babel}
\usepackage[T1]{fontenc}    
\usepackage{hyperref}       
\usepackage{url}            
\usepackage{booktabs}       
\usepackage{amsfonts}       
\usepackage{nicefrac}       
\usepackage{microtype}      
\usepackage{graphicx}

\usepackage{amsmath,amsfonts,bm}

\def\eqref#1{equation~\ref{#1}}

\def\1{\bm{1}}

\def\vs{{\bm{s}}}

\DeclareMathAlphabet{\mathsfit}{\encodingdefault}{\sfdefault}{m}{sl}
\SetMathAlphabet{\mathsfit}{bold}{\encodingdefault}{\sfdefault}{bx}{n}

\newcommand{\E}{\mathbb{E}}

\newcommand{\sigmoid}{\sigma}

\DeclareMathOperator*{\argmax}{arg\,max}

\usepackage{mathtools}
\usepackage{dsfont}
\usepackage{booktabs}
\usepackage{xfrac}
\usepackage{bbm}
\usepackage{booktabs}
\usepackage{algorithm}
\usepackage{algorithmic}
\usepackage[colorinlistoftodos]{todonotes}

\usepackage[most]{tcolorbox}
\usepackage{xparse}
\usepackage{lipsum}
\usepackage{changepage}
\usepackage{enumitem}
\usepackage{graphicx}
\usepackage{xcolor}
\usepackage{wrapfig}
\usepackage{amsmath,amssymb,amsfonts}

\usepackage{caption}
\usepackage{subcaption}

\newif\ifcomments
\commentsfalse
\commentstrue  
\ifcomments
\newcommand{\cg}[1]{\textcolor{purple}{\bf\small [#1 --CG]}}

\newcommand{\mati}[1]{\textcolor{green}{\bf\small [#1 --MM]}}
\definecolor{darkblue}{RGB}{0, 0, 83}
\newcommand{\kmh}[1]{\textcolor{darkblue}{\bf\small [#1 --KMH]}}
\newcommand{\razp}[1]{\textcolor{magenta}{\bf\small [#1 --razvan]}}
\newcommand{\razpt}[1]{\textcolor{magenta}{\bf\small #1}}
\newcommand{\md}[1]{\textcolor{pink}{\bf\small [#1 -- MD]}}
\newcommand{\ar}[1]{\textcolor{orange}{\bf\small [#1 -- ali]}}
\newcommand{\yc}[1]{\textcolor{blue}{\bf\small [#1 --YC]}}

\else
\newcommand{\cg}[1]{}
\newcommand{\mati}[1]{}
\newcommand{\kmh}[1]{}
\newcommand{\razp}[1]{}
\newcommand{\razpt}[1]{}
\newcommand{\md}[1]{}
\newcommand{\ar}[1]{}
\newcommand{\yc}[1]{}
\newcommand{\todo}[1]{}
\fi

\newcommand{\squishlist}{
   \begin{list}{$\bullet$}
    { \setlength{\itemsep}{0pt}      \setlength{\parsep}{3pt}
      \setlength{\topsep}{3pt}       \setlength{\partopsep}{0pt}
      \setlength{\leftmargin}{1.5em} \setlength{\labelwidth}{1em}
      \setlength{\labelsep}{0.5em} } }

\newcommand{\squishlisttwo}{
   \begin{list}{$\bullet$}
    { \setlength{\itemsep}{0pt}    \setlength{\parsep}{0pt}
      \setlength{\topsep}{0pt}     \setlength{\partopsep}{0pt}
      \setlength{\leftmargin}{2em} \setlength{\labelwidth}{1.5em}
      \setlength{\labelsep}{0.5em} } }

\newcommand{\squishend}{
    \end{list}  }

\newcommand{\calA}{{\cal A}}
\newcommand{\calB}{{\cal B}}

\newcommand{\calD}{{\cal D}}

\newcommand{\calL}{{\cal L}}

\newcommand{\calR}{{\cal R}}
\newcommand{\calS}{{\cal S}}

\newcommand{\data}{\calD}

\newcommand{\rhoinit}{{\rho_0}}

\newcommand{\trans}{P}

\newcommand{\Qhat}{\hat{Q}}

\newcommand{\grad}{{\nabla}}

\DeclareMathAlphabet{\mathpzc}{OT1}{pzc}{m}{n}

\usepackage{amsmath}

\usepackage{amssymb}
\usepackage{natbib}
\usepackage{multirow}
\usepackage[accepted]{icml2021}

\usepackage{graphicx}

\usepackage{titlesec}

\usepackage{amsthm}
\newtheorem{theorem}{Remark}
\newtheorem{lemma}{Lemma}
\newtheorem{remark}{Remark}

\icmltitlerunning{Regularized Behavior Value Estimation}

\begin{document}

\twocolumn[

\icmltitle{Regularized Behavior Value Estimation}
\begin{icmlauthorlist}
\icmlauthor{Caglar Gulcehre}{deep}
\icmlauthor{Sergio Gómez Colmenarejo}{deep}
\icmlauthor{Ziyu Wang}{google}
\icmlauthor{Jakub Sygnowski}{deep}
\icmlauthor{Thomas Paine}{deep}
\icmlauthor{Konrad \.Zo\l{}na}{deep}
\icmlauthor{Yutian Chen}{deep}
\icmlauthor{Matthew Hoffman}{deep}
\icmlauthor{Razvan Pascanu}{deep}
\icmlauthor{Nando de Freitas}{deep}
\end{icmlauthorlist}
\icmlaffiliation{deep}{DeepMind, London, UK}
\icmlaffiliation{google}{Google Brain, Toronto, Canada}

\icmlcorrespondingauthor{Caglar Gulcehre}{caglarg@google.com}

\icmlkeywords{Machine Learning}
\vskip 0.3in
]
\printAffiliationsAndNotice{}


\begin{abstract}
Offline reinforcement learning restricts the learning process to rely only on logged-data without access to an environment. While this enables real-world applications, it also poses unique challenges. One important challenge is dealing with errors caused by the over-estimation of values for state-action pairs not well-covered by the training data. Due to bootstrapping, these errors get amplified during training and can lead to divergence, thereby crippling learning. 
To overcome this challenge, we introduce Regularized Behavior Value Estimation ($\mathtt{R}$-$\mathtt{BVE}$). Unlike most approaches, which use policy improvement during training, $\mathtt{R}$-$\mathtt{BVE}$ estimates the value of the behavior policy during training and only performs policy improvement at deployment time. Further, $\mathtt{R}$-$\mathtt{BVE}$ uses a ranking regularisation term that favours actions in the dataset that lead to successful outcomes. We provide ample empirical evidence of $\mathtt{R}$-$\mathtt{BVE}$'s effectiveness, including state-of-the-art performance on the RL Unplugged ATARI dataset. We also test $\mathtt{R}$-$\mathtt{BVE}$ on new datasets, from \texttt{bsuite} and a challenging DeepMind Lab task, and show that $\mathtt{R}$-$\mathtt{BVE}$ outperforms other state-of-the-art discrete control offline RL methods.
\end{abstract}

\section{Introduction}
Deep Reinforcement Learning (deep RL) is dominated by the online paradigm, where agents must interact with the environment repeatedly to explore and learn.
This paradigm has attained considerable success on Atari \citep{mnih2015humanlevel}, Go \citep{silver2017mastering}, StarCraft II and Dota 2 \citep{vinyals2019grandmaster, berner2019dota}, and robotics \citep{andrychowicz2020learning}.
However, the requirements of extensive interaction and exploration make these algorithms unsuitable and unsafe for many real-world applications. In contrast, in the offline setting \citep{fu2020d4rl, fujimoto2018addressing,gulcehre2020rl,levine2020offline}, also known as batch RL \citep{ernst2005tree, lange2012batch}, agents learn from a fixed dataset previously logged by other (possibly unknown) agents.
While the offline setting would enable the application of RL to real-world applications, current algorithms tend to perform worse than their online counterparts.
\begin{figure}[t]
    \centering
    \includegraphics[width=1.1\columnwidth]{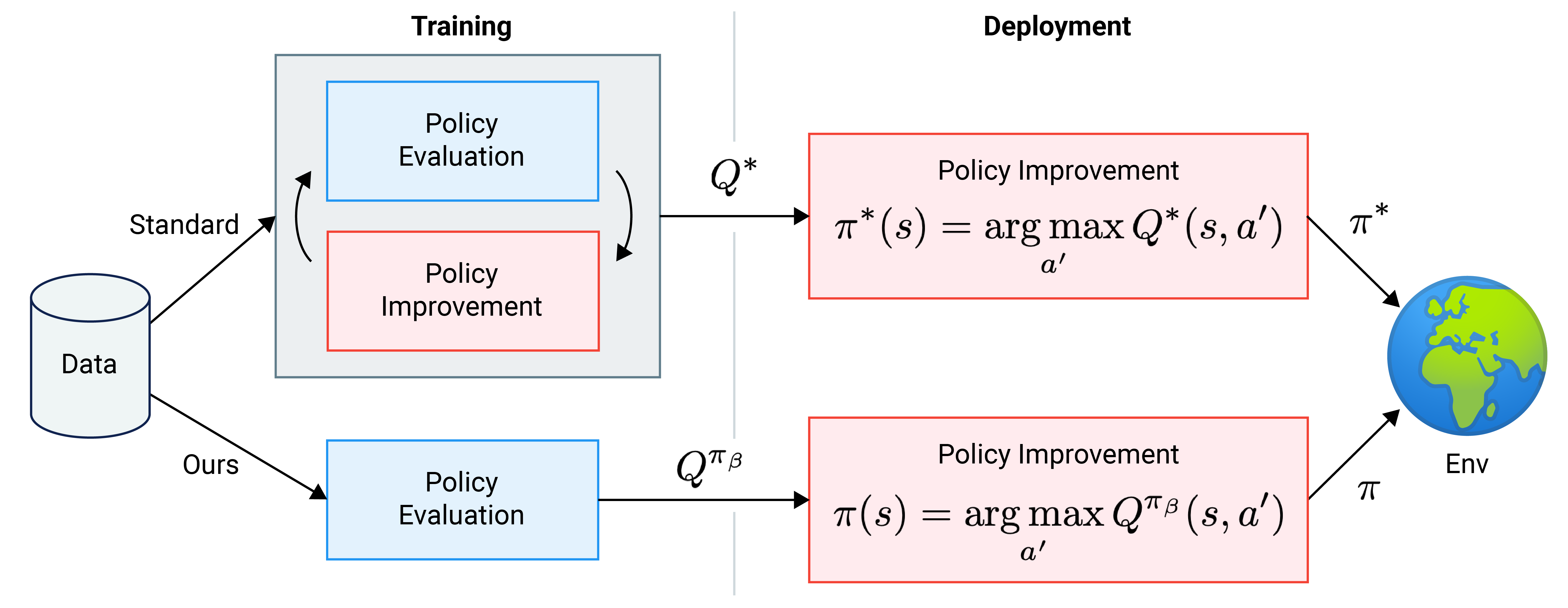}
    \caption{In offline settings, policy improvement is risky as it relies on extrapolating beyond the data. Nevertheless, standard approaches perform many iterations of policy improvement during training, resulting in extrapolation error and divergence. In contrast, \textbf{behavior value estimation} (Ours) estimates the value of the behavior policy, and performs only one step of policy improvement when the policy is deployed. Even one step of policy improvement requires extrapolation, so to make this safer we perform \textbf{ranking regularization} at training time to favour actions in the dataset that lead to successful outcome over unobserved actions. Our approach performs as well or better than state-of-the-art offline RL approaches across several domains.}
    \label{fig:online_v_offline}
\end{figure}

A key difference between online and offline RL is the impact of \emph{over-estimating} the value of unobserved actions (see Figure \ref{fig:over-estimation_underestimation}). In the online case, this over-estimation incentivizes agents to explore actions of high expected reward, thus learning by trial-and-error~\citep{schmidhuber1991possibility}. After trying a new action, the agent is able to update its values.
In contrast, in the offline case, the agent learns from a fixed dataset, and hence the agent does not get the opportunity to interact with the environment to gather new evidence to correct its values. As a result, the agent can become increasingly deluded.

To minimize the harm caused by over-estimation, we propose a new offline RL method that we refer to as \emph{Regularized Behavior Value Estimation} ($\mathtt{R}$-$\mathtt{BVE}$). For clarity of presentation, we introduce the ideas of \emph{behaviour value estimation} and \emph{ranking regularization} separately. We also show in the experiments that our proposed ranking regularization may be applied to improve deep Q-learning methods.

\textbf{Behavior value estimation}. Intuitively, instead of aiming to estimate the optimal policy $Q^{\pi^*}$, we focus on recovering the value of the behavioural policy $Q^{\pi_{\calB}}$. This amounts to removing $\max$-operators during training to prevent over-estimation. Subsequently, to improve upon the behavioral policy, we conduct a single step of policy improvement at deployment time. We found this simple method to work surprisingly well on the offline RL tasks we tried.

\textbf{Ranking regularization}. Behavior value estimation is however not enough to overcome the over-estimation problem.
For this reason, we introduce a $\max$-margin regularizer that encourages actions from rewarding episodes to be ranked higher than any other actions. Intuitively, this regularizer pushes down the value of all unobserved state-action pairs, thereby minimizing the chance of a greedy policy selecting actions under-represented in the dataset. Employing this regularizer during training minimizes the over-estimation impact of the $\max$-operator when used in combination with Q-learning. However, it also minimizes the impact of the $\max$-operator at deployment time, thus making the regularizer useful for behaviour value estimation.
The ranking regularizer can further push down the value of actions that do not appear in rewarding sequences, allowing the learned value function to improve upon the one corresponding to the behaviour policy. 

We evaluate our approach on the open-source RL Unplugged Atari dataset~\citep{gulcehre2020rl}, where we show that $\mathtt{R}$-$\mathtt{BVE}$ outperforms other offline RL methods. We show that $\mathtt{R}$-$\mathtt{BVE}$ performs better on two more datasets: \texttt{bsuite}  \citep{osband2019behaviour} and partially observable \texttt{DeepMind Lab} environments~\citep{beattie2016deepmind}\footnote{We released these datasets under \href{https://github.com/deepmind/deepmind-research/tree/master/rl_unplugged}{RL Unplugged github repo}.}.
We provide careful ablations and analyses that provide insights into our proposed method and existing offline RL algorithms. Empirically, we find out that $\mathtt{R}$-$\mathtt{BVE}$ reduces the over-estimation from extrapolation by orders of magnitude, and improves sample-efficiency significantly (see Figures~\ref{fig:over-estimation_q_values} and~\ref{fig:data_size_perf}).

\section{Background and Problem Statement}
\label{sec:background}

We consider a Markov Decision Process $(\mathcal{S}, \calA, \trans, R, \rhoinit, \gamma)$, where $\mathcal{S}$ is the set of all possible states and $\calA$ for all possible actions. For simplicity, we focus on discrete actions, though our methods can be extended to continuous actions. 
An agent starts in state $s_{0} \sim \rhoinit(\cdot)$ where $\rhoinit(\cdot)$ is a distribution over $\calS$ and takes actions according to its policy $a \sim \pi(\cdot|s)$, $a \in \calA$, when in state $s$.
Then it observes a new state $s'$ and reward $r$ according to the transition distribution $\trans(s'|s,a)$ and reward function $r(s,a)$.

The state-action value function $Q^{\pi}$ describes the expected discounted return starting from state $s$ and action $a$ and following $\pi$ afterwards:
\begin{align}
Q^{\pi}(s, a) &= \mathbb{E}\Bigg[\sum_{t=0}\gamma^t r(s_t, a_t)\Bigg],\\
s_0 = s, a_0 &= a, s_t \sim \trans(\cdot | s_{t-1}, a_{t-1}), a_t \sim \pi(\cdot | \vs_t), 
\end{align}
and $V^{\pi}(s) = \mathbb{E}_{a \sim \pi(\cdot | \vs)}{[Q^{\pi}(s, a)]}$ is the state value function. In this work, we are particularly interested in the scenario where neural networks are used as function approximators to estimate these functions due to their ability to scale and work on complex tasks.

The optimal policy $\pi^*$, which we aim to discover through RL, is one that maximizes the expected cumulative discounted rewards, or \emph{expected returns} such that $Q^{\pi^*}(s, a) \geq Q^{\pi}(s, a) \; \forall s, a, \pi$. 
For notational simplicity, we denote the policy used to generate an offline dataset as $\pi_{\calB}$.
For a state $s$ in an offline dataset, we use $\hat{V}^{\pi_{\calB}}(s)$ to denote the empirical estimate of $V^{\pi_{\calB}}(s)$ computed by summing future discounted rewards over the trajectory that $s$ is part of. 

RL algorithms can be broadly categorized as either \emph{on-policy} or \emph{off-policy}.
Whereas on-policy algorithms update their current policy based on data generated by itself, off-policy approaches can take advantage of data generated by other policies.
Algorithms in the mold of fitted Q-iteration
make up many of the most popular approaches to deep off-policy RL~\citep{mnih2015humanlevel, lillicrap2015continuous, haarnoja2018soft}.
This class of algorithms learns a $Q$ function by minimizing the Temporal Difference (TD) error. To increase stability and sample efficiency, experience replay is typically employed.
For example, DQN~\citep{mnih2015humanlevel} minimizes the following loss function with respect to $\theta$:
\begin{equation}
\calL_{\theta^{\prime}}(\theta) = \mathbb{E}_{(s, a, r, s^{\prime}) \sim \calD}[(Q_{\theta}(s, \hspace{-.5mm}a) - r - \gamma \max_{a^{\prime}}Q_{\theta^{\prime}}(s^{\prime}\hspace{-.5mm}, \hspace{-.5mm}a^{\prime}))^2],
\label{eqn:dqn}
\end{equation}
where $Q_{\theta^{\prime}}(s^{\prime}, a^{\prime})$ is a slowly changing target network, and  
$\calD$ is the replay dataset generated by a behaviour policy. Typically, for off-policy algorithms the behavior policy is periodically updated to remain close to the policy being optimized. A deterministic policy can be derived by defining $\pi(s) = \argmax_a Q(s, a)$.
Various extensions have been proposed, including variants with continuous actions~\citep{lillicrap2015continuous, haarnoja2018soft}, distributional critics \citep{bellemare2017a}, prioritized replays~\citep{schaul2015prioritized}, and n-step returns~\citep{kapturowski2018recurrent, barth-maron2018distributional,hessel2017rainbow}.

In the offline RL setting, agents learn from fixed datasets generated via other processes, thus rendering off-policy RL algorithms particularly pertinent.
Many existing offline RL algorithms adopt variants of Equation~(\ref{eqn:dqn}) to learn value functions; e.g. \cite{agarwal2019striving}.
Offline RL, however, is different from off-policy learning in the online setting. The offline RL datasets are usually finite and fixed, and does not track the policy being learned. When a policy moves towards a part of the state space not covered by the behavior policy(s), for example, one cannot effectively learn the value function. 
We will explore this in more detail in the next~subsection.

\begin{figure}[t]
    \centering
    \includegraphics[width=0.8\columnwidth]{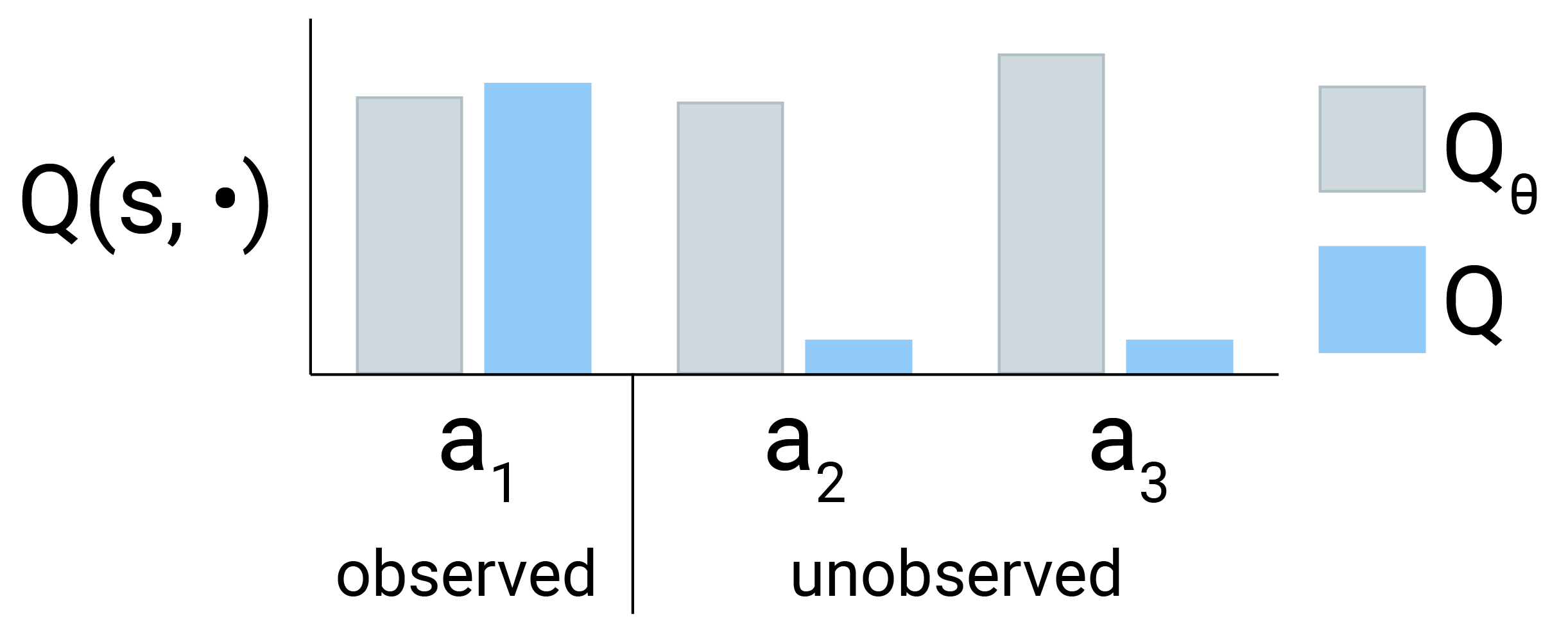}
    \caption{\textbf{Over-estimation from extrapolation.} $Q$ is the real value of an action $a$ in a state $s$, and $Q_{\theta}$ is the estimated value using function approximation. In this example $a_1$ is observed in the dataset while $a_2$ and $a_3$ unobserved. Applying the max operator would cause us to choose $a_3$. We refer to this phenomenon as over-estimation from extrapolation. While in online RL it leads to optimistic exploration followed by correction, in the offline case it results in growing value estimates and divergence. Regularized Behavior Value Estimation is designed to avoid over-estimation from extrapolation during training and deployment.
    \label{fig:over-estimation_underestimation}}
\end{figure}

\subsection{Extrapolation and over-estimation in offline RL}
\label{sec:extrapolation_error}
In the offline setting, when considering all possible actions for the next state in Equation~(\ref{eqn:dqn}), some of the actions will be \emph{out-of-distribution}
(OOD). This is because some actions are never picked in that state by the behavior policy used to construct the training set. That is, the state-action pairs are not present in the dataset. In such circumstances, we have to rely on the current Q-network to extrapolate beyond the training data, resulting in extrapolation errors when evaluating the loss\footnote{Neural networks trained by gradient descent typically struggle to fit under-represented modes of the data. Hence this argument also holds for rarely observed state-action pairs.}. Importantly, this extrapolation can lead to value over-estimation, as explained below. 

Value over-estimation (see Figure~\ref{fig:over-estimation_underestimation}) happens when the function approximator predicts a larger value than the ground truth. In short, taking the max over actions of several Q-network predictions, as in  Equation~(\ref{eqn:dqn}), leads to overconfident estimates of the true value of the state that are propagated in the learning process. 

For OOD actions, we depend on extrapolated values provided by the model $Q_{\theta'}$. We focus on neural networks as the function approximators. While neural networks are efficient learners, they will produce erroneous predictions on unobserved state-action pairs. Sometimes, these will be arbitrary high. These errors will be propagated in the value of other states via bootstrapping. 
Due to the smoothness of neural networks, by increasing the value of actions in the OOD state-action's neighborhood, the overestimated value itself might increase, creating a vicious loop. We remark that, in such a scenario, typical gradient descent optimization can diverge and escape towards infinity. See Appendix~\ref{sec:Theorem} for a formal statement and proof on this statement, though similar observations have been made before by ~\citet{fujimoto2019off} and \citet{achiam2019towards}. When the agent overestimates some state-action pairs in the online setting, they will be chosen more often. This leads to optimistic exploration, which happens both in the on-policy and off-policy setting. The distinction is that the behavior policy trails the learned one instead of being the same for the latter. The online agent would then act, collect data, and correct extrapolation errors. This form of self-correction is absent in the offline setting, and due to the over-estimation from extrapolation, the results can be catastrophic. Additionally, OOD actions' impact becomes larger in the low data regime, where the chance of observing OOD actions is greater and neural networks are more prone to extrapolation errors. 

We make a final note on the hardness of the setup we are interested in. To solve our problems, we require powerful function approximators that make analytical studies difficult without strong assumptions that might not hold in practice. Additionally, the tasks are complex large-scale tasks, which can make assumptions on the coverage of the state-action space by the training set improbable or at least impossible to assess. Therefore we mainly rely on empirical evidence to motivate our approach.
\section{Regularized Behavior Value Estimation}
\label{sec:solutions}

\subsection{Behavior Value Estimation ($\mathtt{BVE}$)}
\label{sec:bve} 

One particular starting point to deal with the vicious loop caused by over-estimation could be simply to phrase the problem as a supervised one, where given the data, a neural network is trained to produce for a given state the observed action. For the discrete action setting, the one we are focusing on, this is a well defined classification problem that will mimic the behaviour policy used to collect the data. 

However, offline RL aims to improve upon the behavior policy. To do so, many algorithms borrow machinery from the online scenario, as for example, Equation (\ref{eqn:dqn}) which through the $\max$-operator tries to improve the current policy by greedily picking the action with the largest estimated value. The repeated application of this improvement step can cause learning to diverge in the offline regime, as discussed earlier. Therefore, to break this cycle, we propose to limit the number of 
improvement steps being done, bounding the impact of over-estimation. In particular, given that policy iteration algorithms typically do not require more than a few steps to converge to optimal policy \citep[Chapter~4.3]{lagoudakis2003least, sutton2018reinforcement}, we focus on the extreme scenario where we are allowed to take a single improvement step. 

We start by drawing inspiration from the SARSA update~\citep{rummery1994line, van2009theoretical}. SARSA relies on the consecutive application of a \emph{policy evaluation} step, followed by a \emph{policy improvement} step. The policy evaluation step is described by the update below:
\begin{equation}
\calL_{\theta^{\prime}}(\theta) = \mathbb{E}_{(s, a, r, s^{\prime}, a^{\prime}) \sim \pi}(Q_{\theta}(s, \hspace{-.5mm}a) \hspace{-.5mm} - \hspace{-.5mm}r \hspace{-.5mm}- \hspace{-.5mm}\gamma Q_{\theta^{\prime}}(s^{\prime}\hspace{-.5mm}, \hspace{-.5mm}a^{\prime}))^2.
\label{eqn:sarsa}
\end{equation}
The policy improvement step is \emph{implicit} by defining the policy $\pi$ used to collect data greedily in terms of $Q_\theta$.
Therefore we can remove the policy improvement step by simply eliminating this dependency of $\pi$ on $Q_\theta$. We propose \emph{Behaviour Value Estimation} as applying Equation~(\ref{eqn:sarsa}) on transitions  $(s,a, r, s^\prime, a^\prime)$ collected by the behaviour policy, leading to a Q-function reflecting the value of this policy.

Finally, we take a single \emph{policy improvement} step at deployment by acting greedily with respect to $Q_\theta$, namely  $\pi(s) = \argmax_a Q_{\theta}(s, a)$. A single improvement step can not 
guarantee convergence to an optimal behaviour in general. However, our approach will effectively bound the impact of extrapolation 
errors and empirically provides significant improvement on the tasks considered, see for example Figure~\ref{fig:dmlab_perfs}. This is particularly true in the low data regime where neural networks are less reliable and repeated policy improvement steps are prone to lead to an amplified over-estimation of OOD actions, as observed empirically in Figure~\ref{fig:over-estimation_q_values}. To highlight the efficiency of a single improvement step, we provide through the lemmas in Appendix~\ref{sec:optimal_policy_conditions} some sufficient conditions for a single improvement step to lead to optimal behaviour. 
\subsection{Ranking regularization ($\mathtt{R}$-$\mathtt{}$)}
Behaviour Value Estimation effectively reduces over-estimation during training by avoiding policy improvement steps but it also avoids estimating the $Q$ values of OOD actions and focuses on recovering the $Q$ values of the observed data under the behaviour policy. In contrast to the tabular case, where all OOD actions will have a default value of $0$, neural networks will extrapolate based on the observed~data.

We choose to learn a more conservative $Q$ function that avoids picking actions not seen during training to mitigate the impact of the extrapolation. To this end, here, we simply regularize the output of the Q-network to rank the observed actions higher than the OOD ones for successful episodes. 

This ranking strategy can be formulated with a hinge-loss~\citep{chen2009ranking,burges2005learning}. We adopt the squared-hinge loss which is commonly used for RankSVM \citep{chapelle2010efficient}. Given a transition from the dataset $\left(s_t, a_t\right)$ we introduce the following loss:
\begin{equation}
    \mathcal{C}(\theta) = \hspace{-2mm} \sum_{i=0,i \neq t}^{|\calA|} \hspace{-2mm} \max\left(Q_{\theta}(s_t,a_i) - Q_{\theta}(s_t,a_t) + \textcolor{red}{\nu}, 0\right)^2,
    \label{eqn:quirk_reg_main}
\end{equation}
which states that if the value of any action $a_i$ is higher by a a certain margin than the value of the dataset action $a_t$, then the value function needs to be adjusted. 

Previously, hinge losses were used for regularization in RL, but with different goals from ours \citep{su2020conqur,pohlen2018observe}.

Blindly applying this ranking loss on suboptimal data can have an undesirable effect. This ranking loss mostly focuses on the frequency of an action in the dataset, promoting frequent actions to have larger value than infrequent ones. For suboptimal behaviour policies is likely that frequency does not correlate with high values.

One practical answer to this problem is to filter out  trajectories that are not sufficiently rewarding, reducing the frequency of state-actions that are not relevant for good behaviour. We adopt a soft filtering approach, which simply re-weights each transition with the normalized value of the trajectory:
\begin{align}
    \textcolor{blue}{\mathtt{w}(s)} &= \exp((\hat{V}^{\pi_{\calB}}(s) - \E_{s\sim \data}[\hat{V}^{\pi_{\calB}}(s)])/\textcolor{cyan}{\beta})\\ \label{eqn:quirk_reg_final}
    \mathcal{C}(\theta) &=  \textcolor{blue}{\mathtt{w}}(s) \hspace{-2mm} \sum_{i=0,i \neq t}^{|\calA|} \hspace{-2mm} \max\left(Q_{\theta}(s,a_i) - Q_{\theta}(s,a_t) + \textcolor{red}{\nu}, 0\right)^2, \nonumber
\end{align}
where again $\hat{V}^{\pi_{\calB}}(s)$ denotes the empirical estimate of $V^{\pi_{\calB}}(s)$ computed by summing future discounted rewards over the trajectory that $s$ is part of. The expectation
 $\E_{s\sim \data}[\hat{V}^{\pi_{\calB}}(s)]$ is estimated with an average over mini-batches. 
 
The weight $\textcolor{blue}{\mathtt{w}(s)}$ can be thought of as a success filter. The formulation of this filtering mechanism is based on the advantage based filtering mechanism used for the policy in CRR \citep{wang2020critic}. In our preliminary experiments, we tried both binary filtering with an indicator function and advantage based filtering, but we found that the filtering mechanism introduced above works the best on discrete control algorithms we study here.  In the online regime, the filter amplifies the choice of actions in trajectories that fair better than expected, behaving similarly to a policy improvement step. In particular, our loss is closely related to Ranking Policy Gradient~\citep{DBLP:conf/iclr/LinZ20}. While the same machinery of this work cannot be used to prove convergence of our loss due to the offline character of our regime, it provides intuition of why the filter biases the learning in the right direction. Compared with a typical policy improvement step, this loss, however, also ensures that OOD actions (and actions in non-rewarding sequences) rank lower in the policy.

The ranking loss can be successfully combined with Q-learning, which we will refer as $\mathtt{R}$-$\mathtt{DQN}$ or, better with behaviour value estimation, $\mathtt{R}$-$\mathtt{BVE}$. It is due to this reweighting that $\mathtt{R}$-$\mathtt{BVE}$ can not only further mitigate over-estimation compared to $\mathtt{BVE}$ but it can also improve on the behaviour policy during training. In all our experiments, we fixed $\beta$ to be $0.5$, a hyperparameter value borrowed from CRR \citep{wang2020critic}.  

\section{Related work}

Early examples of offline/batch RL include least-squares temporal difference methods~\citep{bradtke1996linear,lagoudakis2003least} and fitted Q iteration~\citep{ernst2005tree, riedmiller2005neural}. More recent offline RL approaches fall into three broad categories: 1) \textbf{policy-constraint} approaches regularize the learned policy to stay close to the behavior policy either explicitly or implicitly \citep{fujimoto2019off, kumar2019stabilizing, jaques2019way, siegel2020keep, wang2020critic, ghasemipour2020emaq}, 2) \textbf{value-based} approaches lower value estimates for unseen state-actions pairs, either through regularization or uncertainty \citep{kumar2020conservative, agarwal2019optimistic}. 3) \textbf{model-based} approaches similarly lower reward estimates for unseen state-action pairs \cite{yu2020mopo, kidambi2020morel}. 

The above methods perform policy improvement during training, and use the methods described above to mitigate extrapolation error and divergence. Our approach is fundamentally different. We perform no policy improvement during training, instead estimating the value of the behavior policy. This approach is related to policy-constraint methods but has three advantages: 1) we do not need to estimate a behavior policy, 2) we do not constrain our policy to be similar to the behavior policy, which may be a poor constraint when the behavior policy is a mixture of many policies some of which are sub-optimal and 3) we are guaranteed to avoid extrapolation errors during training. Additionally, our method uses value-based regularization but only to mitigate extrapolation error when we deploy our policy. Our value-based regularization is most similar to regularization used in CQL \citep{kumar2020conservative}. There are two main differences 1) our regularization is weighted by the success of the trajectory and 2) we use a max-margin ranking loss. We find these work better in practice.

\section{Experiments}

We investigate the performance of offline RL algorithms with discrete actions across three domains: \texttt{bsuite}, Atari, and DeepMind Lab. The main question we want to answer is: how do algorithms perform when there is low coverage of state-action pairs? In that context, we study multiple factors that affect coverage including dataset size (see Figure~\ref{fig:over-estimation_q_values} and \ref{fig:data_size_perf}), noise (see Figure~\ref{fig:bsuite_experiments} and~\ref{fig:dmlab_stochasticity_env}), and partial observability (Figure~\ref{fig:dmlab_perfs}).

We compare our proposed model $\mathtt{R}$-$\mathtt{BVE}$ (Regularized Behavior Value Estimation), its ablations $\mathtt{BVE}$ (Behavior Value Estimation) and $\mathtt{R}$-$\mathtt{DQN}$ (Regularized Deep Q-learning), and several offline reinforcement learning baselines. For the Atari domain we used the RL Unplugged Atari benchmark \citep{gulcehre2020rl}. We also create additional datasets using \texttt{bsuite} to have fast diagnostic tests, and DeepMind Lab to have a challenging partially observable domain. We are working to open-source these datasets. The details of the datasets are provided in Appendix \ref{sec:datasets}. In all our experiments, for our CQL baseline, we used CQL($\mathcal{H}$) with a fixed $\alpha$ as prescribed in \cite{kumar2020conservative} for Atari. 

Overall, our method has two hyperparameters that we tuned: $\textcolor{red}{\nu}$ for the margin, and the regularization coefficient $\lambda$ for the ranking regularization. Our method seems to be quite robust to those values. We tuned $\lambda$ only on Atari, and found out that $0.005$ worked best on nine Atari online policy selection games. We used this value for all other tasks and games. We fixed $\textcolor{red}{\nu}=0.05$ for both Atari and \texttt{bsuite} without any hyperparameter tuning.  In particular, these hyperparameters fared well  for the offline policy selection games. This provides an evidence that, for a real-world application, one might be able to tune the regularization hyperparameters in a simulated version of the environment and deploy them in the setting of interest. However, on DeepMind Lab, based on a coarse grid for the margin hyperparameter $\textcolor{red}{\nu} \in \{0.5, 0.05\}$, we found some variation for the optimal value. On our \texttt{bsuite} and DeepMind Lab plots, we report median and standard error bars across different seeds. On Atari, as accustomed in the literature, we only report median across different Atari games.

\subsection{\texttt{bsuite} Experiments}
\label{sec:bsuite}

\texttt{bsuite} \citep{osband2019behaviour} is a proposed benchmark designed to highlight key aspects of an agent's scalability such as exploration, memory or credit assignment. We generated low-coverage offline RL datasets for {\it catch}, {\it mountain\_car} and {\it cartpole} by recording the experiences of an online agent during training, as described by \cite{agarwal2019optimistic}, and then subsampling it (see Appendix \ref{sec:bsuite_dataset} for details.) We generated two versions of datasets for each task: stochastic vs. deterministic. The stochastic data is obtained by injecting noise into transitions by replacing an agent's actions with a random action with probability $\epsilon=0.25$.

On \texttt{bsuite} experiments, our goals are: firstly to have a computationally cheap setting to compare our proposed methods against the state-of-the-art baselines on diagnostic tasks and, secondly, verify that our method is robust against the stochasticity in the transitions.
\begin{figure}[t]
    \centering
    \includegraphics[width=\columnwidth]{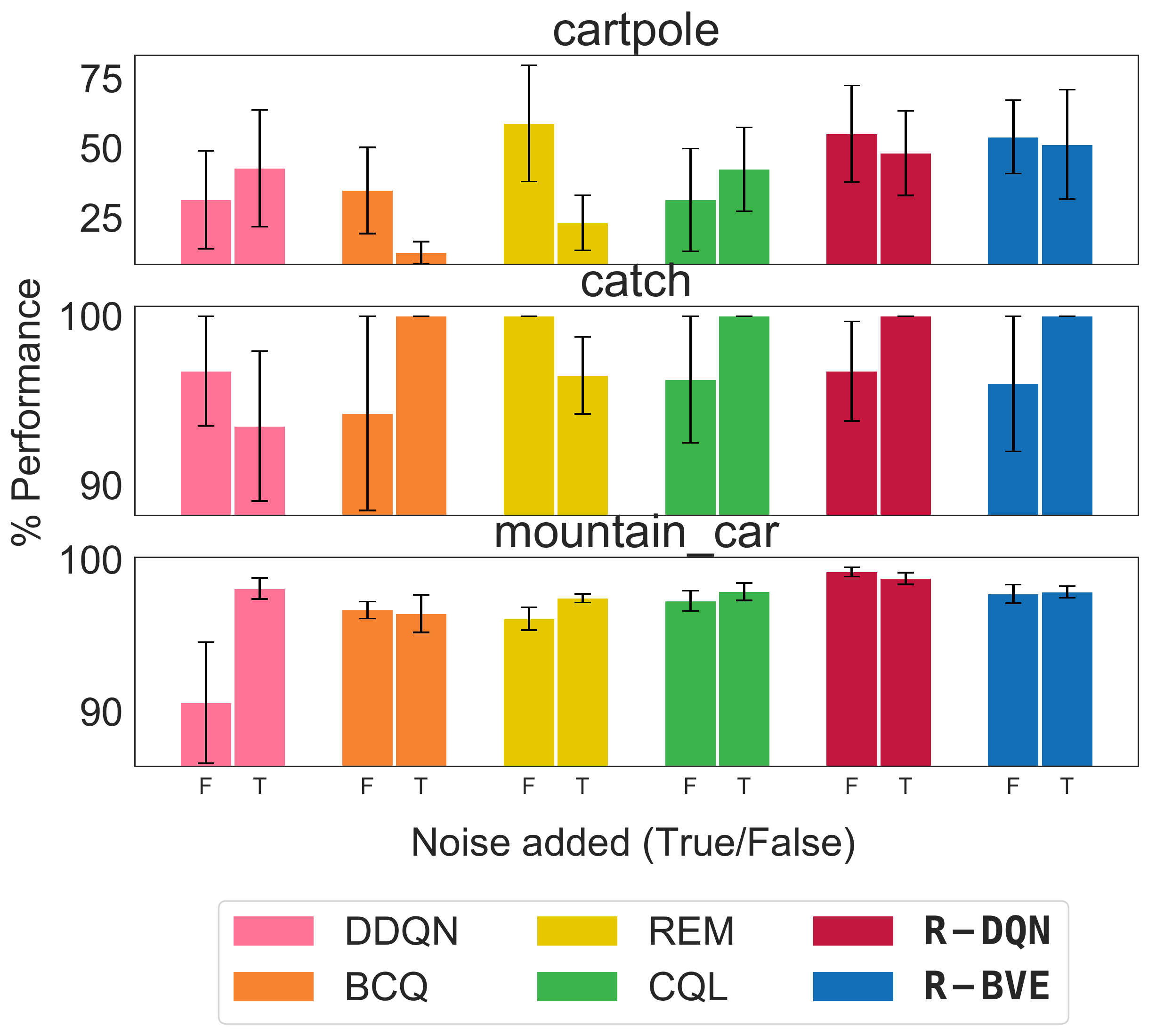}
    \caption{{\bf Bsuite Experiments.} \texttt{bsuite} results on three environments without, and with noise injected into the actions in the dataset. The proposed method, $\mathtt{R}$-$\mathtt{BVE}$, outperforms all the baselines on {\it cartpole}. Methods implementing a form of behavior constraining (BCQ, CQL and our methods $\mathtt{R}$-$\mathtt{BVE}$ and $\mathtt{R}$-$\mathtt{DQN}$) excel on {\it catch} (with noise), stressing its importance.}
    \label{fig:bsuite_experiments}
\end{figure}

In Figure~\ref{fig:bsuite_experiments}, we compare the effectiveness of $\mathtt{R}$-$\mathtt{BVE}$ and $\mathtt{R}$-$\mathtt{DQN}$ against four baselines: DDQN \citep{hasselt2010double}, CQL($\mathcal{H}$) \citep{kumar2020conservative}, REM \citep{agarwal2019optimistic} and BCQ \citep{fujimoto2018addressing}. On the harder dataset ({\it cartpole}), $\mathtt{R}$-$\mathtt{BVE}$ and $\mathtt{R}$-$\mathtt{DQN}$, our proposed methods, outperform all other approaches on the noisy datasets showing the efficiency and robustness of our approach. Two other methods, REM and CQL, also perform relatively well in the noisy setting. The results for {\it catch} are similar, with the exception that BCQ also has better normalized score which re-emphasises the importance of restricting behavior to stay close to the observed data.
On {\it mountain\_car}, all methods almost reaches the performance of the expert generated the dataset on the noisy dataset. However, DDQN performed poorly when there is no noise in the transitions. We think, this might be because the noise injected into the dataset although increases stochasticity, it also increases the (state, action)-coverage of the dataset. Overall, $\mathtt{R}$-$\mathtt{BVE}$ and $\mathtt{R}$-$\mathtt{DQN}$ seem to be less effected by this noise injected.

\subsection{Atari Experiments}
\label{sec:atari_experiments}

Atari is an established online RL benchmark \citep{bellemare2013arcade}, which has recently attracted the attention of the offline RL community \citep{agarwal2019optimistic,fujimoto2019benchmarking}. Here, we used the experimental protocol and datasets from the RL Unplugged Atari benchmark \citep{gulcehre2020rl}. We report the median normalized score across the Atari games as prescribed by \cite{gulcehre2020rl}.

\begin{figure}[t]
    \centering
    \includegraphics[width=\columnwidth]{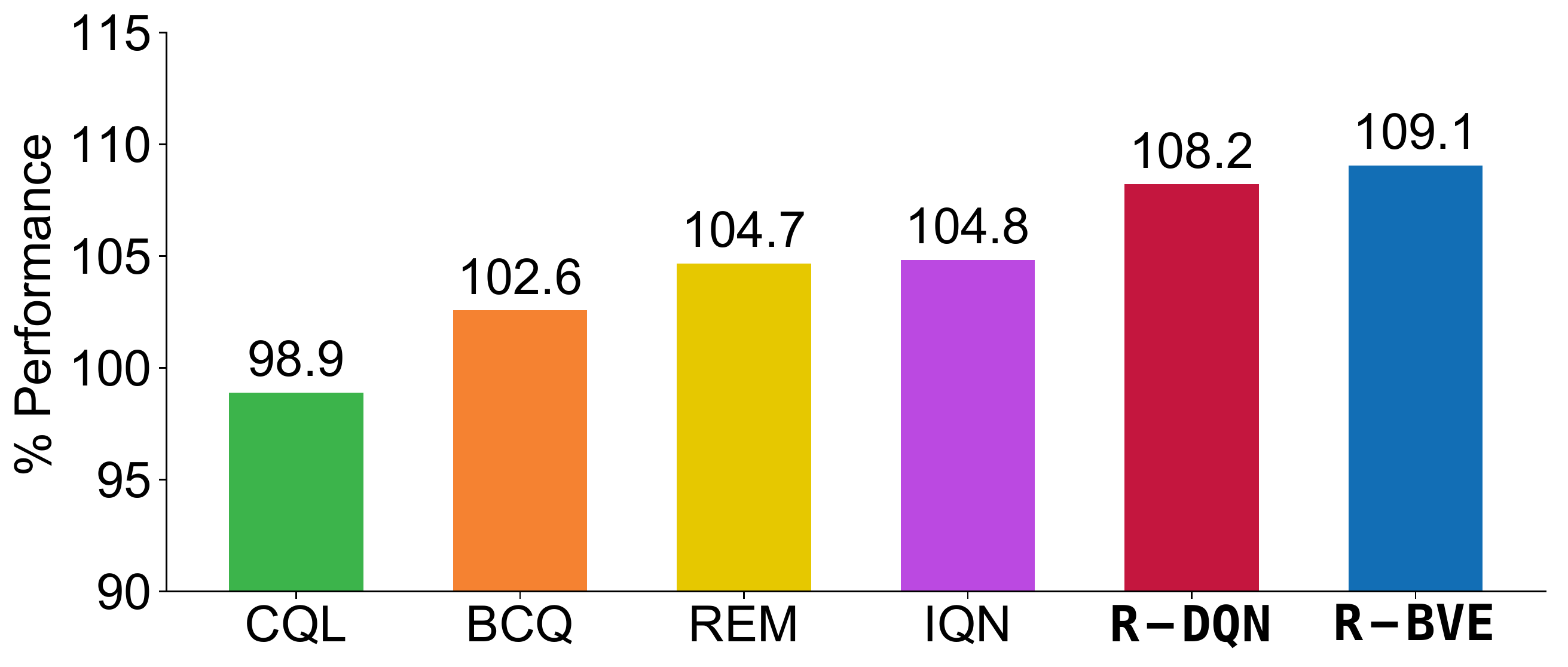}
    \caption{{\bf Atari offline policy selection games results.} We compare our proposed $\mathtt{R}$-$\mathtt{BVE}$ against other recent state-of-the-art offline RL algorithms with hyperparameters tuned only on offline policy selection games. Our methods prove to be the best.}
    \label{fig:atari_testing_results}
\end{figure}
In Figure~\ref{fig:atari_testing_results}, we show that $\mathtt{R}$-$\mathtt{BVE}$ outperforms all baselines reported in the RL Unplugged benchmark as well as CQL($\mathcal{H}$) \citep{kumar2020conservative} on offline policy selection games. Both $\mathtt{R}$-$\mathtt{DQN}$ and $\mathtt{R}$-$\mathtt{BVE}$ outperform other SOTA offline RL methods. This experiment highlights an important point: on a large benchmark suite (including 37 Atari games) a single step of policy improvement is sufficient to outperform several offline RL algorithms including policy contraint (BCQ), value-based regularization (CQL), and value-based uncertainty methods (REM, IQN). Though we should note, in this setting there is enough data for the neural networks to learn reasonable approximations of the Q-value (exploiting the structure of the state space to extrapolate reasonably for unobserved state-action pairs.) This figure also highlights the robustness of ranking regularization to its hyperparameters since we did not do any hyperparameter search on the offline policy selection games.

\begin{figure}[t]
  \includegraphics[width=\columnwidth]{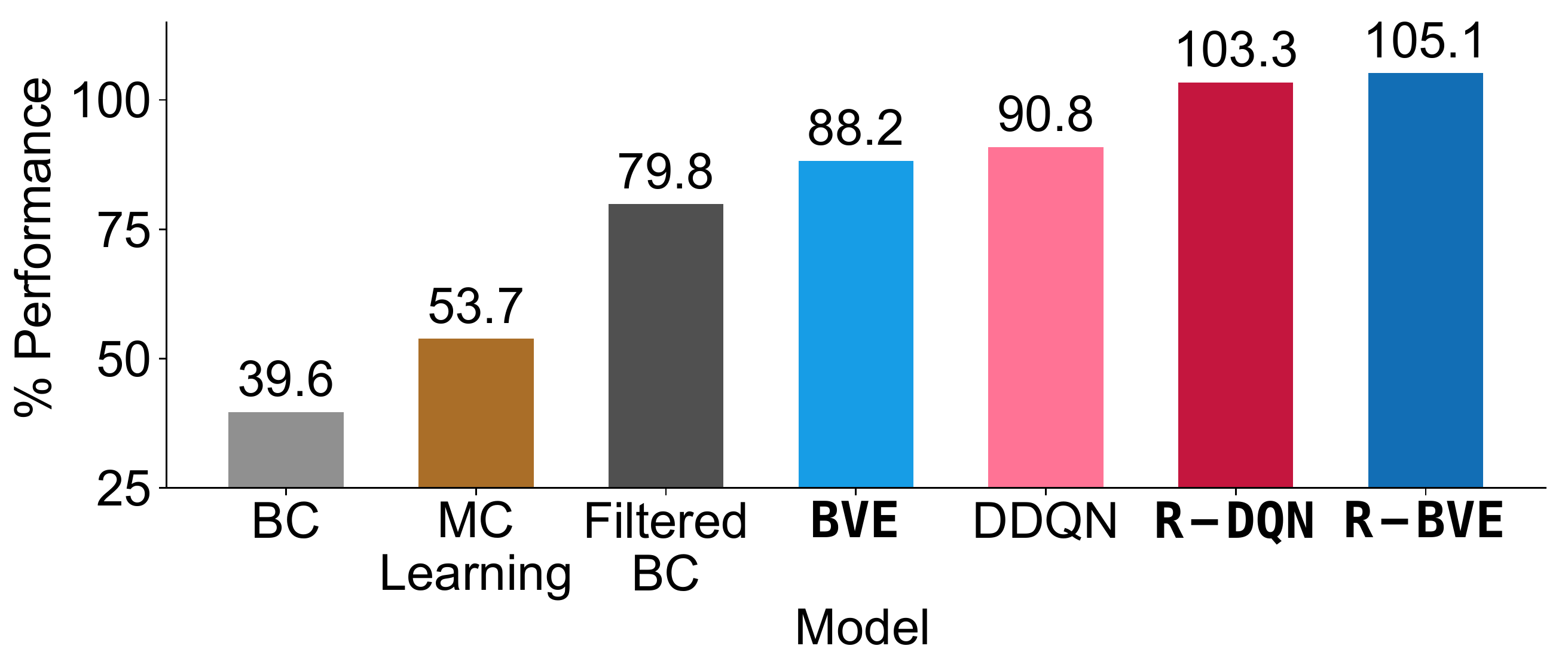}
  \caption{{\bf Atari ablations (online policy selection games).} We compare $\mathtt{BVE}$, DDQN and Monte Carlo Learning, and ranking regularization in RLU Atari dataset for online policy selection games. $\mathtt{BVE}$ and DDQN attain similar median episodic returns, but learning with Monte Carlo returns performs poorly. We found that the most significant improvement from the ranking regularization term ($\mathtt{R}$-$\mathtt{DQN}$ and $\mathtt{R}$-$\mathtt{BVE}$). Although the combination of ranking regularization and $\mathtt{BVE}$ performs the best.
  \label{fig:QUIRK_ablations}}
\end{figure}
\vspace{-2mm}
\paragraph{Ablation Experiments on Atari}
We ablate two different aspects of our algorithm on nine \textbf{online policy selection games} from RL Unplugged Atari suite: i) the choice of TD backup updates (Q-learning or behavior value estimation), ii) the effect of ranking regularization. We show the ablation of those three components in Figure~\ref{fig:QUIRK_ablations}. We observed the largest improvement when using ranking regularization. In general, we found that directly using Monte-Carlo estimation for the value function (we refer this in our plot as "MC Learning") does not work on Atari. We also provide results with behavior cloning (BC) \citep{pomerleau1989alvinn} and filtered BC --- BC that is only trained on highly rewarding episodes. These are episodes for whom their episodic return is greater than a thereshold, where we set this threshold  to be the mean of episodic return in the whole dataset. We showed that BC on the whole dataset works very poorly, but filtered BC works considerably better. Nevertheless, filtered BC still performs considerably worse than other offline RL methods such as $\mathtt{BVE}$.

\begin{figure}[t]
    \centering
    \includegraphics[width=\columnwidth]{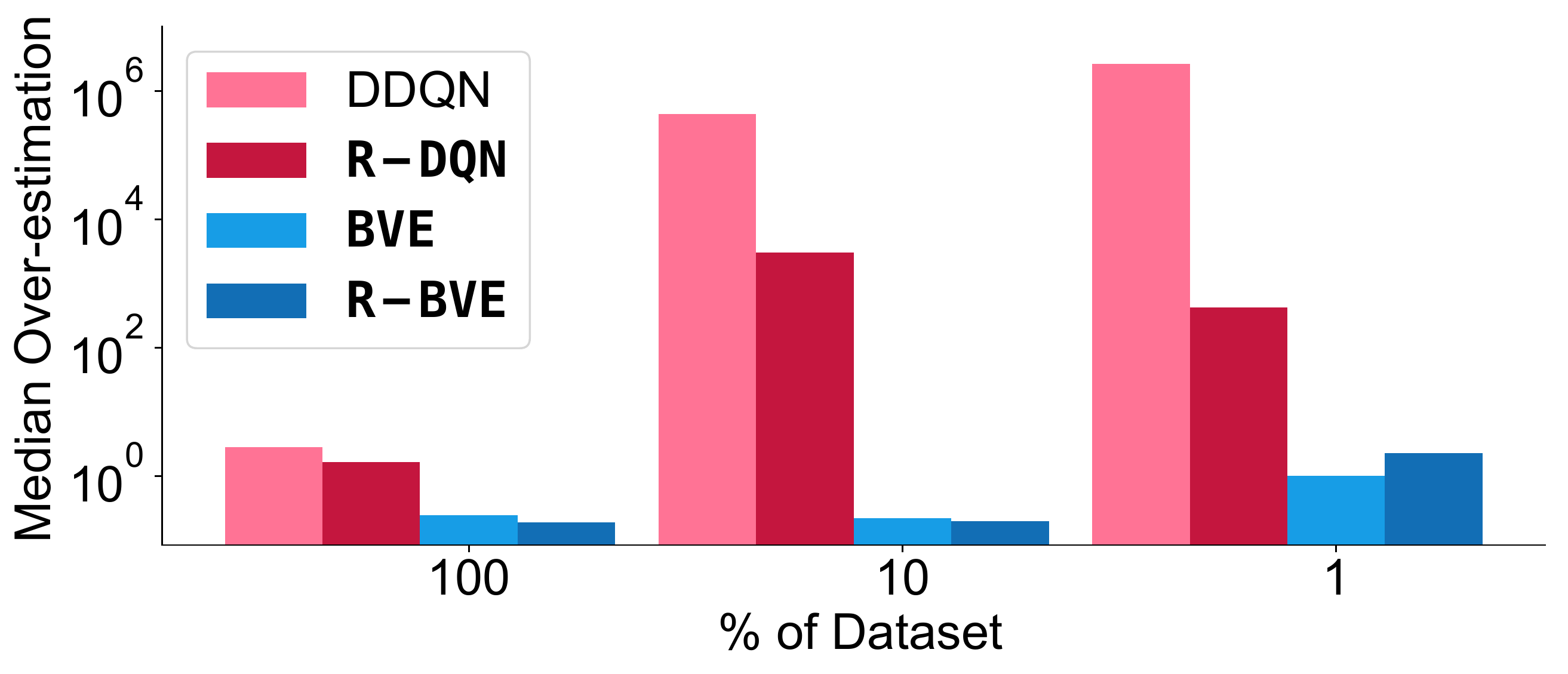}
    \caption{\textbf{Over-estimation of Q-values.} We report median over-estimation error over online policy selection games on Atari with sub-sampled Atari datasets (\% of dataset). As dataset size decreases DDQN is prone to severe over-estimation. $\mathtt{R}$-$\mathtt{DQN}$ reduces this to some extent but still over-estimates severely. Finally, $\mathtt{BVE}$ and $\mathtt{R}$-$\mathtt{BVE}$ significantly reduce over-estimation.
    \label{fig:over-estimation_q_values}}
\end{figure}
\paragraph{Over-estimation Experiments}
One source for over-estimation of Q-learning is the maximization bias~\cite{hasselt2010double}. In the offline setting, another source, as discussed in Section~\ref{sec:background}, is due to extrapolation errors.  Double DQN (DDQN by \citet{hasselt2010double}) is supposed to address the first problem, but it is unclear whether it can address the second. In Figure~\ref{fig:over-estimation_q_values}, we show that in the offline setting DDQN still over-estimates severely when we evaluate the critic's predictions in the environment. This suggests the second source, dominates and DDQN does not address it. In comparison, $\mathtt{R}$-$\mathtt{DQN}$ overestimates significantly less, highlighting the efficiency of the ranking regularization, however its performance degrades by many orders of magnitude in the low data regime.   Finally, $\mathtt{BVE}$  and $\mathtt{R}$-$\mathtt{BVE}$
drastically reduce over-estimation in all settings. 
In the figure, we compute the over-estimation error by evaluating the methods in the environment and computing $\frac{1}{100}\sum_{i=0}^{100}(\max(Q^{\pi}(s, a) - G^{\pi}(s), 0))^2$ over $100$ episodes, where $G^{\pi}(s)$ corresponds to the discounted sum of rewards from state $s$ till the end of episode by following the policy~$\pi$.
\begin{figure}[t]
    \centering
    \includegraphics[width=\columnwidth]{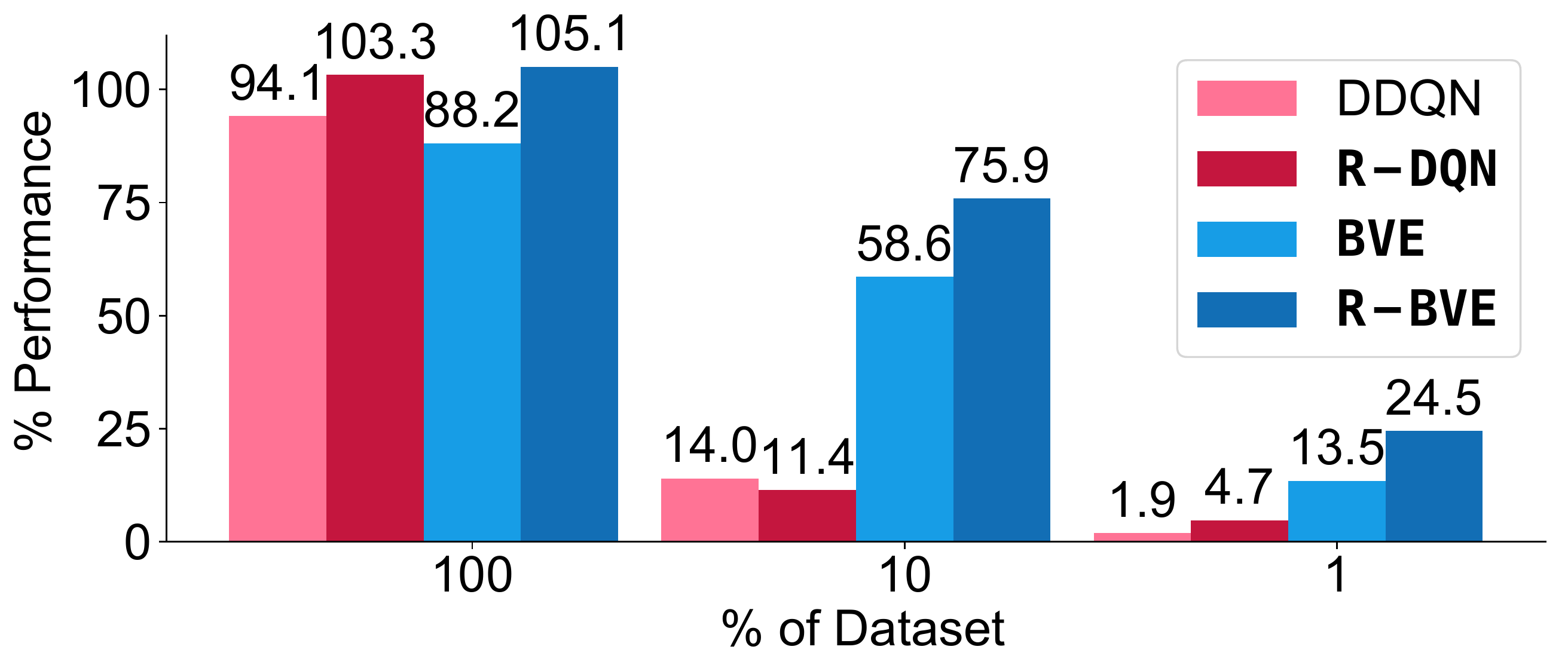}
    \caption{{\bf Sample Efficiency.} Normalized scores of DDQN, \mbox{$\mathtt{R}$-$\mathtt{DQN}$}, $\mathtt{BVE}$ and $\mathtt{R}$-$\mathtt{BVE}$ on subsets of data from online policy selection games. On the smaller subsets of the dataset $\mathtt{BVE}$ and $\mathtt{R}$-$\mathtt{BVE}$ performs comparatively better than DQN and $\mathtt{R}$-$\mathtt{DQN}$. The Q-learning suffers more since the coverage of the dataset reduces with the subsampling causing more severe extrapolation error.
    \label{fig:data_size_perf}}
\end{figure}

\paragraph{Robustness Experiments}

According to Figure \ref{fig:QUIRK_ablations}, $\mathtt{BVE}$ and Q-learning both have similar performance on the full dataset, however, in low data regimes, the behavior value policy considerably outperforms Q-learning (see Figure~\ref{fig:data_size_perf}). The poor performance of DQN and $\mathtt{R}$-$\mathtt{DQN}$ in the lower data regime is potentially due to the over-estimation that we showed in Figure \ref{fig:over-estimation_q_values}. Furthermore, in Appendix \ref{sec:apx:atari_robustness} (see Figure \ref{fig:SARSA_vs_DQN_reward_filtering}), we investigate the robustness of $\mathtt{BVE}$ and DDQN with respect to the reward distribution. We found that the performance of $\mathtt{BVE}$ is more robust than DDQN to  variations of the reward distribution.

\begin{figure*}[htbp!]
    \centering
    \includegraphics[width=\linewidth]{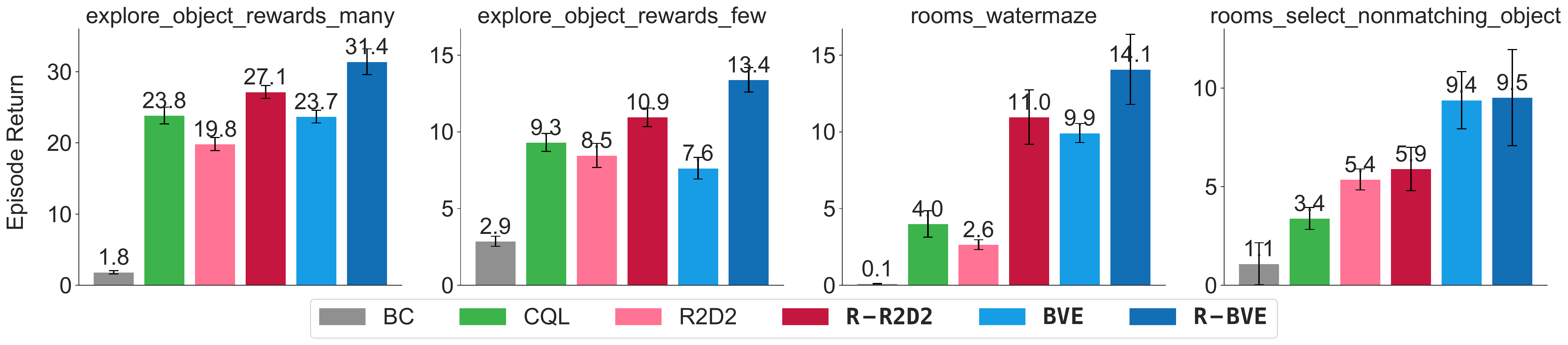}
    \caption{{\bf DeepMind Lab Results:} We compare the performance of different baselines on challenging DeepMind Lab datasets coming from four different DeepMind Lab levels. Our method, $\mathtt{R}$-$\mathtt{BVE}$, consistently performs the best.}
    \label{fig:dmlab_perfs}
\end{figure*}

\subsection{DeepMind Lab Experiments}
\label{sec:dmlab}

Offline RL research mainly focused on fully observable environments such as Atari. However, in a complex partially observable environment such as Deepmind Lab, it is very difficult to obtain good coverage in the dataset even after collecting billions of transitions. 
To highlight this, we have generated datasets by training an online R2D2 agent on DeepMind Lab levels. Specifically, we have generated datasets for four of the levels: $\mathtt{explore\_object\_rewards\_many}$, $\mathtt{explore\_object\_rewards\_few}$, $\mathtt{rooms\_watermaze}$, and $\mathtt{rooms\_select\_nonmatching\_object}$. The details of the datasets are provided in the Appendix~\ref{sec:dmlab_dataset}.

We compare offline BC, CQL, R2D2, $\mathtt{R}$-$\mathtt{R2D2}$, $\mathtt{BVE}$ and $\mathtt{R}$-$\mathtt{BVE}$ on our DeepMind Lab datasets. We use the same architecture, so the main difference is in the loss function. To test our hypothesis that partially observable environments accentuate the coverage issue we consider the large data regime, using 300M transitions stored during the online training of R2D2. In Figure \ref{fig:dmlab_perfs}, we show the performance of each algorithm on four different levels. Our proposed modifications, $\mathtt{R}$-$\mathtt{BVE}$ and $\mathtt{BVE}$ outperform other offline RL approaches on all DeepMind Lab levels. We argue that poor performance of R2D2 in the offline setting is due to the implicit low coverage of the dataset. Even at 300M transitions do not seem enough, potentially due to the partially observable nature of the environment as well as its diversity. 

\paragraph{The importance of coverage in offline RL: }

Here, we investigate the effect of coverage on the DeepMind Lab $\mathtt{seekavoid\_arena\_01}$ level on a dataset generated by a fixed policy. To do so, we generated two datasets.
The first relies on a R2D2 snapshot, as the behaviour policy, which we refer to as \emph{Expert Data}. The second uses a noisy version of this policy, where $75\%$ of decisions are taken by the R2D2 agent, and the remaining $25\%$ rely on a uniform policy over actions, referred as \emph{Noisy Expert Data}. 
Figure \ref{fig:dmlab_stochasticity_env} depicts the result on these two datasets.
BC outperforms all offline RL approaches on the noiseless expert dataset. However, introducing the noise into the actions deteriorates BC's performance considerably, and $\mathtt{R}$-$\mathtt{BVE}$ outperforms it. 
Over all, the episodic returns obtained by offline RL methods either improve (R2D2, $\mathtt{BVE}$,  $\mathtt{R}$-$\mathtt{BVE}$) or stay unaffacted (CQL, $\mathtt{R}$-$\mathtt{R2D2}$)  on the noisy expert data compared to their performance on the expert data. This is most likely due to the fact that the noisy expert data provides better coverage. Given that the environment is deterministic, the expert data follows mostly the optimal trajectory offering a relative low coverage of the state-action space.

\begin{figure}[t]
\vspace{-0.3cm}
    \centering
    \includegraphics[width=\columnwidth]{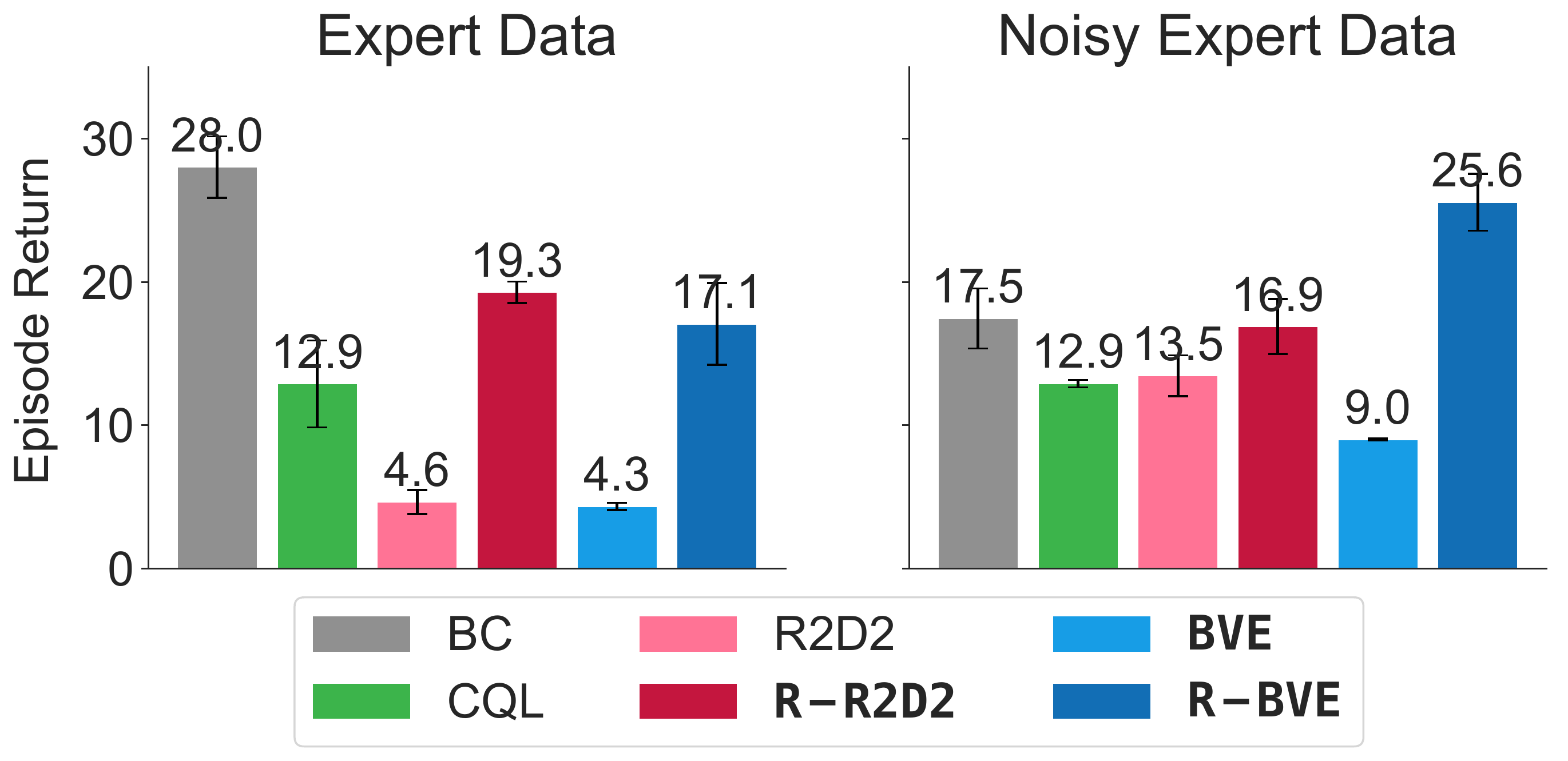}
    \caption{{\bf Effect of coverage in the expert demonstrations:} Given access to data from an optimal policy only, BC performs well, where as offline RL methods struggle due to insufficient state-action coverage. Adding noise to the policy, in a way simulating the human demonstrations, increases coverage. In this setting BC performs worse, but R-BVE is able to match the performance of BC from optimal data.}
    \label{fig:dmlab_stochasticity_env}
\end{figure}

\section{Discussion}

In this work we investigate the deep offline RL setting with discrete actions. In such settings, overestimation errors of the neural network approximator get propagated in the value function through the bootstraping. This can lead to a cycle, where the overestimation gets amplified with every policy improvement step, in the worst case scenario leading these values to escape to infinity. We propose behaviour value estimation $\mathtt{BVE}$ an algorithm that limits the number of policy improvements steps to one. Empirically, we showed that this single policy improvement step at deployment is enough for the relatively complex tasks we have considered.\\

Behavior value estimation is however not enough to overcome the overestimation problem. For this reason, we introduce a max-margin regularizer that encourages actions from rewarding episodes to be ranked higher than any other actions. This leads to the proposed algorithm $\mathtt{R}$-$\mathtt{BVE}$ which outperforms all the other SOTA offline RL methods we have compared against for the discrete control on \texttt{bsuite}, RLU Atari and DeepMind Lab tasks. We note that our algorithm is particularly effective for low data regimes, where the overestimation issue discussed in this work is acute. 

Finally, we have also proposed two new offline RL datasets: (i) DeepMind Lab: large scale, partially observable environment, (ii) bsuite: small scale, low data regime. We believe these datasets poses unique characteristics  and that they will be of interest to the community. 
As future work, we plan to extend $\mathtt{R}$-$\mathtt{BVE}$ to continuous control and other real-world applications.


{
\fontsize{9.5}{11}\selectfont
\bibliography{refs}
\bibliographystyle{plainnat}
}

\clearpage
\appendix
\onecolumn
\title{Appendix}
\section{Q-learning can escape to infinity in the offline case}
\label{sec:Theorem}

\begin{theorem}
 Q-learning, using neural networks as a function approximator, can diverge in the offline RL setting given that the collected dataset does not include all possible state-actions pairs, even if it contains all transitions along optimal paths. Furthermore, the parameters (and hence the Q-values themselves) can diverge towards infinity under gradient descent dynamics. 
\end{theorem}

\begin{proof}
The proof relies on providing a particular instance where Q-learning diverges towards infinity. This is \emph{sufficient} to show that divergence can happen. Note that the remark does not make any statement of how likely is for this to happen, nor is providing sufficient conditions under which such divergence has to happen. 
 
 Let us consider a simple deterministic MDP depicted in the figure below (left).
 
 \begin{center}
 \includegraphics[trim=0 250 300 0,clip,width=\columnwidth]{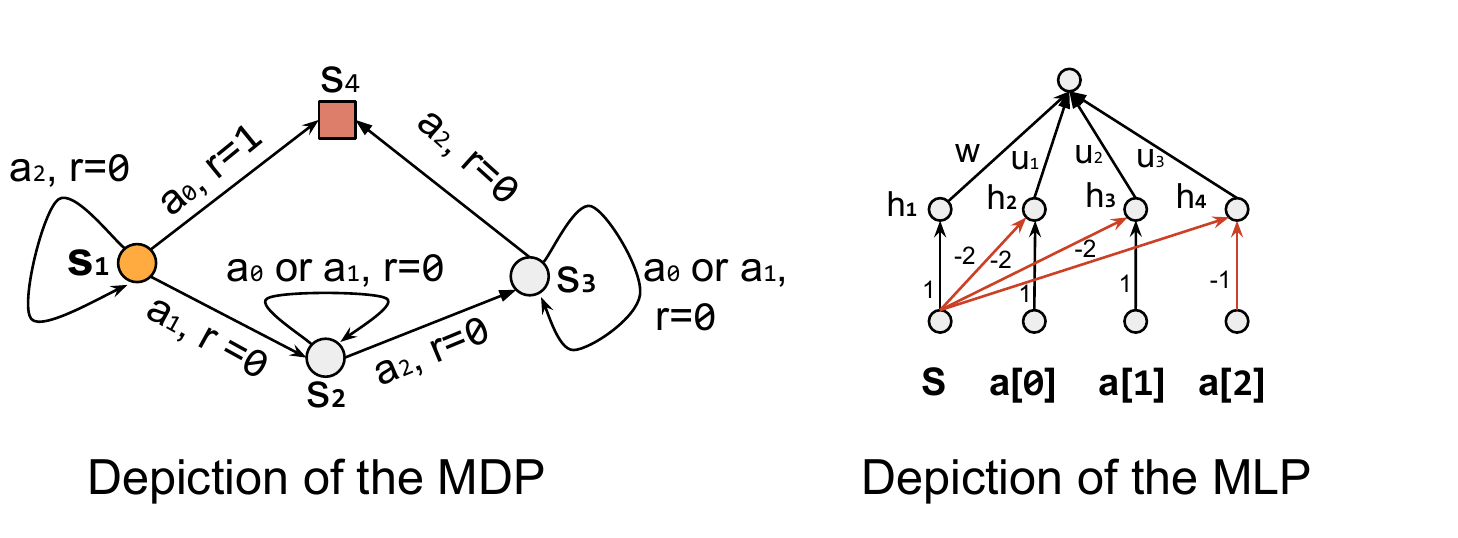}
 \end{center}

 $\mathcal{S}=\{s_1, s_2, s_3, s_4\}$ is the set of all states, where $S_1$ is deterministically the starting state and $S_4$ is the terminal state of the MDP. Let $\calA=\{a_0,a_1,a_2\}$ be the set of all possible actions. Let the reward function $r(s,a)$ be  $0$ for all action-state pair except $r(s_1, a_0)$ which is 1.
 Let the transition probabilities $\trans(s'|s,a)$ be deterministic as defined by the depicted arrows. I.e. for any state action pair only transitioning to one state has probability $1$, while the rest has probability $0$. For example, only   $\trans(s_4|s_1,a_0)=1$, while  $\trans(s_3|s_1,a_0)=0, \trans(s_2|s_1,a_0)=0. \trans(s_2|s_1,a_0)=0$. For $s_1, a_1$ only 
 $\trans(s_2|s_1,a_1)=1$ and so on and so forth. 
 
 \emph{First observation} is that the optimal behavior is to pick action $a_0$ (as it is the only rewarding transition in the entire MDP).

 The features describing each state are given by a single real number, where $s_1=0, s_2 =1, s_3=\beta$, with $\beta > \frac{1}{\gamma} > 0$, where $\gamma$ is the discount factor. Assume actions are provided to the neural network as one-hot vectors, i.e. $a_0=[0, 0, 1]^T, a_1=[0, 1, 0]^T, a_2=[1,0,0]^T$\footnote{one-hot representation is the typical representation for action in discrete spaces}, where we will refer to $a[i]$ as the $i$-th element of the vector that represents the action $a$. For example $a_0[2]=1$ and $a_0[0]=0$. 
 
 Let us consider the Q-function parametrized as a simple MLP (depicted in the figure above left). The MLP uses rectifier activations, and gets as input both the state and action, returning a single scalar value which is the Q-value for that particular state action combination. Rewriting the diagram in analytical form we have that for $s\in \mathcal{R}$ and $a\in\mathcal{R}^3$:
 
 \begin{equation}
     Q_\theta(s,a) = w\cdot relu(s) + u_1 relu(a[0]-2s) + u_2 relu(a[1]-2s) + u_3 relu(-2s -a[2])
 \end{equation} 
 
 \emph{A note on initialization.} The weights of the first layer are given as constants. The process would work if we leave them to be learnable as well, but the analysis would become considerably harder. The exact value used, $-2, 1, -1$, are not important. In principle we care for the negative weights connecting $s$ to $h_2, h_3, h_4$ be larger in magnitude than those from $a[i]$ to $h_i$, and we care for the weight between $a[2]$ and $h_2$ to be negative. They can be scaled arbitrarily small and do not need to be identical. 
 
 What we will rely in the rest of the analysis is that the preactivation of $h_2, h_3, h_4$ to be negative for state $s_2$ and $s_3$. 
 This will be in the zero region of the  rectifier, meaning no gradient will flow through those units. Since $s_3 > s_2 \geq 1$ and $a[i] \in \{0,1\}$, it is sufficient for the weight from $s$ to $h_2, h_3, h_4$ to be larger in magnitude than the weight from $a[i]$ to $h_2, h_3, h_4$. This ensures that for $s>1$, the Q-function is not a function of $u_i$ as $u_i$ will get multiplied by $0$.\footnote{The fact that no gradient gets propagated in the first layer is only important if we attempt to consider the case when the first layer weights are learnable.} Also we want the function to never depend on $u_3$ to simplify our analysis, which is easily achievable if the weight going from $a[2]$ to $h_4$ is negative. 
 
 Given the observations above, if we plug in the formula the different values of $s_i$ and $a_i$ we get that:
 
 \begin{equation}
 \begin{array}{lll}
     & Q_\theta(s_1, a_0) & = u_1 \\
     & Q_\theta(s,1, a_1) &= u_2 \\
     & Q_\theta(s_1, a_2) &= 0 \\
     \forall a \in \mathcal{A}, & Q_\theta(s_2, a) &= w  \\
     \forall a \in \mathcal{A}, & Q_\theta(s_3, a) &= \beta w 
     \end{array}
     \label{eq:individualQvals}
 \end{equation}
 
 Note that this implies that 
 \begin{equation}
 \begin{array}{l}
     \max_a Q_\theta(s_2,a) = w \\
     \max_a Q_\theta(s_3,a) = \beta w
\end{array}
     \label{eq:maxQtoy}
 \end{equation}

 Assume $w>0$. And let the dataset collected by the behavior policy to contain the following 3 transitions:

 $$\mathcal{D}=\{ (s_1, a_0, 1, s_4), (s_1, a_1, 0, s_2), (s_2, a_2, 0, s_3)\}$$
 
 We can now construct the Q-learning loss that we will use to learn the function $Q$ in the offline case which will be 
 
 \begin{equation}
  \begin{array}{ll}
     \mathcal{L} & = \sum_{(s,a,r,s')\in\mathcal{D}} \left(Q_\theta(s,a) - r - \gamma max_a Q_\theta'(s',a)\right)^2 \\
      & = (Q_\theta(s_1, s_0) - 1)^2 + (Q_\theta(s_1, a_1) - \gamma \max_a Q_\theta'(s_2,a))^2 + (Q_\theta(s_2, a_2) - \gamma \max_a Q_\theta'(s_3, a))^2 \\
      & = (u_1 - 1)^2 + (u_2 - \gamma w')^2 + (w - \gamma \beta w')^2
        \end{array}
        \label{eq:toyExampleUnfolded}
 \end{equation}
 
 Note that we relied on Equation~(\ref{eq:maxQtoy}) to evaluate the $\max$ operator and $\theta'$ is a copy of $\theta$, that is used for bootstrapping. This is the standard definition of Q-learning see Equation~(\ref{eqn:dqn}). In particular in this toy example $\theta'$ is numerically always identical to $\theta$ (in general it can be a trailing copy of $\theta$ from k steps back) and is used more to indicate that when we take a derivative of the loss with respect to $\theta$ we do not differentiate through $Q_\theta'$. From Equation~(\ref{eq:toyExampleUnfolded}) we notice that only the first transition in dataset contributes to the gradient of $u_1$, only the second transition contributes to the gradient of $u_2$ and only the third transition contributes to the gradient of $w$. 
 We can not evaluate the gradient with respect to $\theta$ of the loss $\mathcal{L}$ over the entire dataset:
 
 \begin{equation}
 \begin{array}{lll}
 \nabla_{u_1} &= u_1 - 1 &\\
 \nabla_{u_2} &= u_2 - (0 + \gamma w) & \\
 \nabla_{u_3} &= w - (0 + \gamma \beta w) &= (1 - \gamma \beta) w \\
 \nabla_{w} &= w - (0 + \gamma \beta w) &= (1 - \gamma \beta) w\\
 \end{array}
 \end{equation}
 
 Note that we assumed $w > 0$ and for simplicity we exploited that $w'=w$ numerically, to be able to better understand the dynamics of the update. Given that $\beta > \frac{1}{\gamma}$, $\nabla_{w}$ will always be negative as long as $w$ (and implicitly $w'$) stays positive. Given that $w_t = w_{t-1} - \alpha \nabla_w$ for some learning rate $\alpha >0$, the update creates a vicious loop that will increase the norm of $w$ at every iterations, such that $\lim_{t \to \infty} w_t = \infty$. Given that the gradient on $u_2$ tracks $w$, it means that the path that takes action $a_2$ in the initial state $s_1$ will have $+\infty$ as value. \emph{Note that all transitions along the optimal path of this deterministic MDP are part of the dataset.} 
 
 Also that given our example, the same will happen if we rely on SGD rather than batch GD (as the different examples affect different parameters of the model independently and there is no effect from averaging). Preconditioning the updates (as for e.g. is done by Adam or rmsprop) will also not change the result as they will not affect the sign of the gradient (the preconditioning matrix needs to be positive definite). Neither momentum will not affect the divergence of learning, as it will not affect the sign of the update. 
 
 This means that the provided MDP will diverge towards infinity under the updates on most commonly used gradient based algorithms. 
 
\end{proof}

\section{The surprising efficiency of 1-step of policy improvement}\label{sec:optimal_policy_conditions}
In this section, we show examples where 1-step of policy improvement is surprisingly effective.
We first start by going over some theoretical results.

\begin{lemma}
Assume a Markov Decision Process that satisfies the following conditions: (1) the transition distribution $P$ is deterministic, (2) any trajectory has a finite length (3) we are interested in the episodic return, i.e., the discount factor $\gamma=1$, (4) the reward function $r(s, a)=0$ for any non-terminating state and $r(s, a) \in \left\{L, H\right\}$ for any terminating state with constants $L < H$. Denote by $\calS$ the set of states from which there exists a trajectory with a reward of $H$. If a behavior policy $\pi_{\calB}$ can generate a trajectory with a reward of $H$ with a non-zero probability for any initial state $s\in \calS$, then one step of policy improvement on the exact behavior value function will give an optimal policy $\pi^*$.
\label{lemma:lemma_mdp_det}
\end{lemma}
\begin{proof}
Let us note that, our proof assumes a tabular case, and by our definition any trajectory receives an episodic return of either $L$ or $H$. A policy $\pi$ is optimal in this environment if and only if it receives a reward of $H$ with a probability of 1 from any state $s \in \calS$. For a behavior policy satisfying the premise, its value function is as follows,
\begin{align}
    Q(s, a) &= \mathbb{E}_{\tau \sim \pi_{\calB}} \left[\sum_{t=0}^{\infty} \gamma^t r(s_t, a_t)\right] = \Big(p(\tau_H|s, a) H + (1 - p(\tau_H|s, a))L\Big) \in \left[L, H\right], \nonumber\\
    V(s) &= \mathbb{E}_{a \sim \pi_{\calB}(s)} [Q(s, a)] \left\{
    \begin{array}{ll}
    > L & \mbox{if $s\in \calS$;} \\
    =L & \mbox{otherwise.}
    \end{array}\right.
\end{align}
where $p(\tau_H|s, a)$ denotes the probability of trajectories sampled from $(s, a)$ with a reward of $H$.
The policy after one step of policy improvement is defined as follows with a tie being broken arbitrarily.
\begin{equation}
    \pi(s) = \argmax_a Q(s, a)
\end{equation}
From any state $s_0\in \calS$, let $a_0 \sim \pi(s_0)$ be any sample from $\pi$. Because $V(s) > L$, we know $Q(s_0, a_0) > L$, which infers there exists a trajectory starting from $(s_0, a_0)$ with a reward of $H$. Let $s_1$ be the next state following action $a_0$ in the deterministic environment. Consequently we know $s_1 \in \calS$. We can apply $\pi$ repeatedly to sample a trajectory $s_0, a_0, s_1, a_1, \dots$. Because $s_t \in \calS, \forall t \geq 0$ by induction and any trajectory is finite according to the premise, any sampled trajectory from $s_0$ will eventually reach a terminating state with a reward of $H$, and therefore, the policy $\pi$ receives an episodic return of $H$ with a probability of 1 from any state $s_0 \in \calS$ and is therefore an optimal policy.
\end{proof}
\begin{remark}
The conditions of the MDP can be satisfied by a broad range of reinforcement learning problems, such as a goal reaching task that receives a constant positive reward only if it reaches the goal within a time limit, or an autonomous driving task that will receive a constant penalty if it crashes in a given time period.
\end{remark}
\begin{remark}
The assumption on the behavior policy assumes a sufficient exploration in the dataset so that a good trajectory exists for any initial state, but it doesn't impose any restriction on the average performance.
\end{remark}

\begin{figure}[htp!]
    \centering
    \includegraphics[width=0.8\columnwidth]{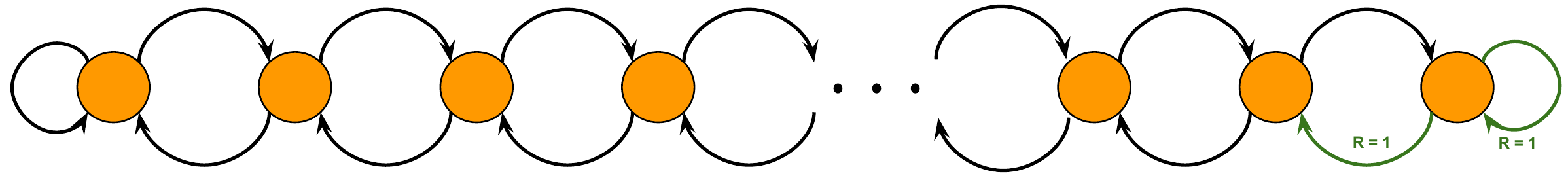}
    \caption{Chain MDP. The initial state is the leftmost state. There are only two rewarding transitions both of which come out of the right most state. All other transitions have a reward of 0. The behavioral policy is uniformly random between the two actions in each state.
    \label{fig:chain}}
\end{figure}

\begin{lemma}
In the chain MDP (Figure ~\ref{fig:chain}) of at least 2 states with discount $\gamma \in (0, 1)$, under uniformly random behavior policy, one step of policy improvement (given the true value function) would lead to an optimal policy.
\end{lemma}
\begin{proof}
We name the states in this class of environment numerically. That is the left most state is state $1$ and the last state $n$.
Denote $V_i$ to be the value function of the uniform random policy at state $i$.
Let $V = [V_0, V_1, \cdots, V_n]$. It is easy to show that $(I - \gamma P)V = R$ where $P$ is the state transition matrix $R = [0, \cdots, 1]^T$.
We know $$
(I - \gamma P) = \left( \begin{array}{cccccc}
1-\frac{1}{2}\gamma & -\frac{1}{2}\gamma &   & & &0 \\
-\frac{1}{2}\gamma & 1 & -\frac{1}{2}\gamma  & & &\\
 &-\frac{1}{2}\gamma & 1 & -\frac{1}{2}\gamma & & \\
&& \ddots & \ddots & \ddots & \\
&&&-\frac{1}{2}\gamma & 1 & -\frac{1}{2}\gamma \\
0&& & & -\frac{1}{2}\gamma & 1-\frac{1}{2}\gamma \\
\end{array} 
\right).
$$

Since all rewards are positive, we know $V_i > 0 \; \forall \, i$.
Since $(1-\frac{1}{2}\gamma)V_1-\frac{1}{2}\gamma V_2 = 0$, $V_1 > 0 \; V_2 > 0$, and $\gamma \in (0, 1)$, we know $V_2 > V_1$.
Assume that $V_k > V_{k-1}$ where $k < n$ , then $V_{k+1} > V_{k}$ since $\frac{1}{2}\gamma(V_{k+1} + V_{k-1}) = V_{k}$.
Therefore, by induction, we show that the values are monotonically increasing. 

In every state, going left along the chain is therefore of lower $Q$ value compared to going right (if going right in the rightmost states entails staying). 
We know the optimal policy in this environment is to go right. Therefore 1 step of policy improvement provides the optimal policy.
\end{proof}

\begin{remark}
In a 2-d grid world, we can show an analogous result. We, however, omit this as the proof follows a similar structure albeit more complicated.
\end{remark}

One-step of policy improvement, despite the examples we have shown in this section, does not always lead to an optimal policy. 
We argue, however, it is surprisingly good in many scenarios. In this section, we present such an example in a 2-d grid world (see Figure \ref{fig:grid_world}). In this example, one-step of policy improvement from a random policy is not sufficient to derive an optimal policy. The resulting policy, however, is very close in performance to the optimal policy when measured by the value of the initial state.
It is reasonable to assume that many environments share the properties for which BVE is an effective technique.

\begin{figure}[htp!]
    \centering
    \includegraphics[width=0.7\columnwidth]{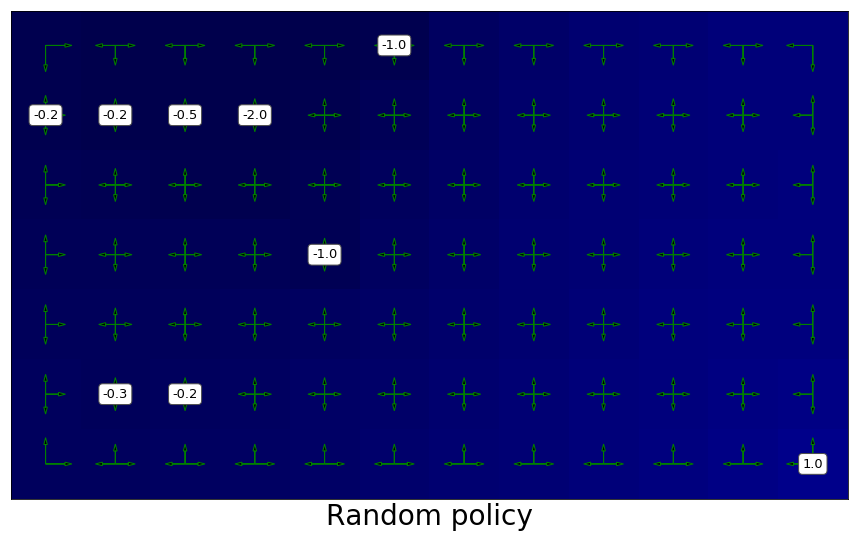}
    \includegraphics[width=0.7\columnwidth]{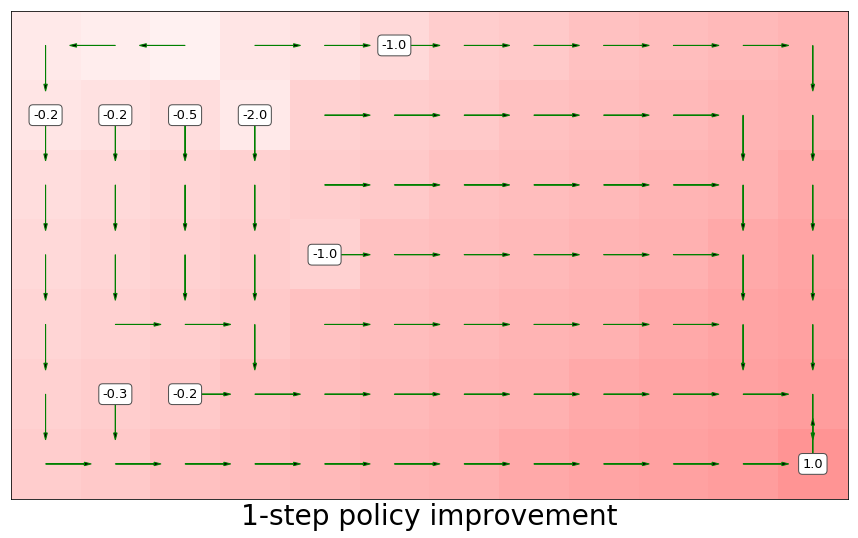}
    \includegraphics[width=0.7\columnwidth]{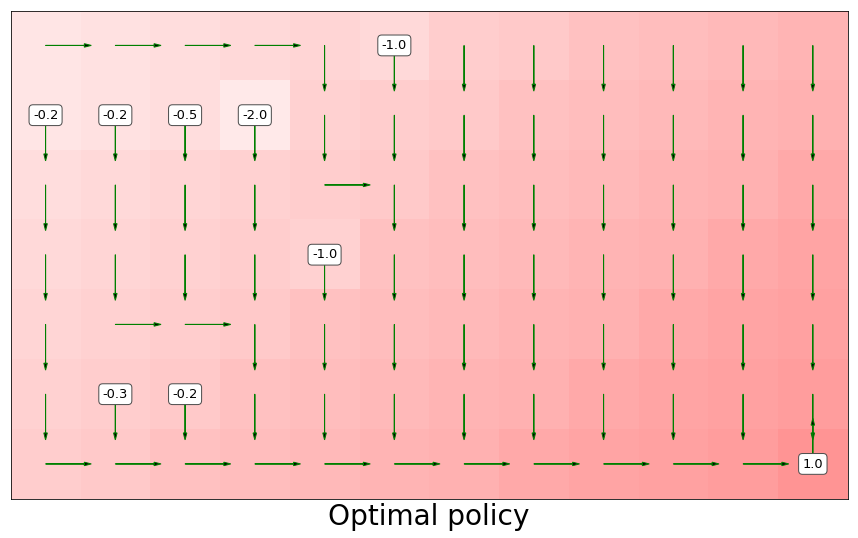}
    \caption{Grid world example. Here, the numbers in the cells represent rewards. (The reward is 0 in the absence of numbers). The green arrows represent the policy. The color of the cell shows the GT value function. [\textsc{Top}] An uniformly random policy achieves value of $-9.518$ at the initial state (upper left state). [\textsc{Middle}] With one step of policy improvement over the random polices, we achieve $42.160$ in value. [\textsc{Bottom}] The optimal policy in this environment obtains $42.359$ in value for the initial state, only $0.2$ higher than the value achieved using only 1 step of policy improvement.
    \label{fig:grid_world}}
\end{figure}

\section{Additional Results and Ablations}
\label{sec:additional_results}

\subsection{On the Effects of Regularization}
\label{sec:effect_of_regularization}
In this section, we study the effect of the regularization on the action gap and the over-estimation error. In Figure \ref{fig:action_gap_regularization}, we show that increasing the regularization co-efficient for the ranking regularization increases the action gap across the Atari online policy selection games which can result to lower estimation error and better optimization.
\begin{figure}[htp!]
    \centering
    \includegraphics[width=0.85\columnwidth]{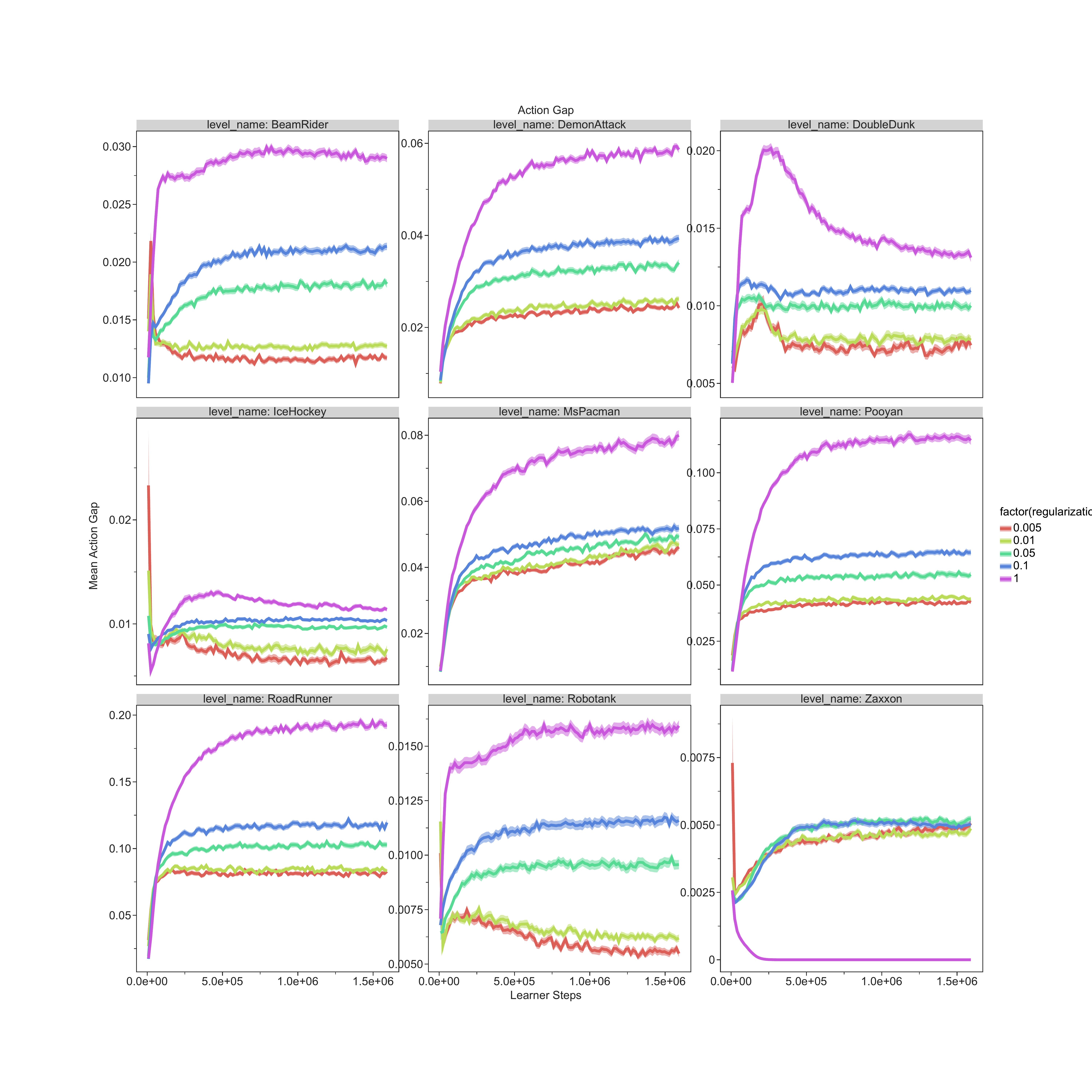}
    \caption{{\bf The effect of increasing the ranking regularization on the action gap.} We can see that increasing the regularization significantly increases the action-gap which potentially can help further stabilize the learning.}
    \label{fig:action_gap_regularization}
\end{figure}

In Figure \ref{fig:over-estimation_regularization}, we investigate the effect of increasing the regularization on the over-estimation of the Q-network when evaluated in the environment. We visualize the mean over-estimation across the online policy selection games for Atari.
\begin{figure}[htp!]
    \centering
    \includegraphics[width=0.85\columnwidth]{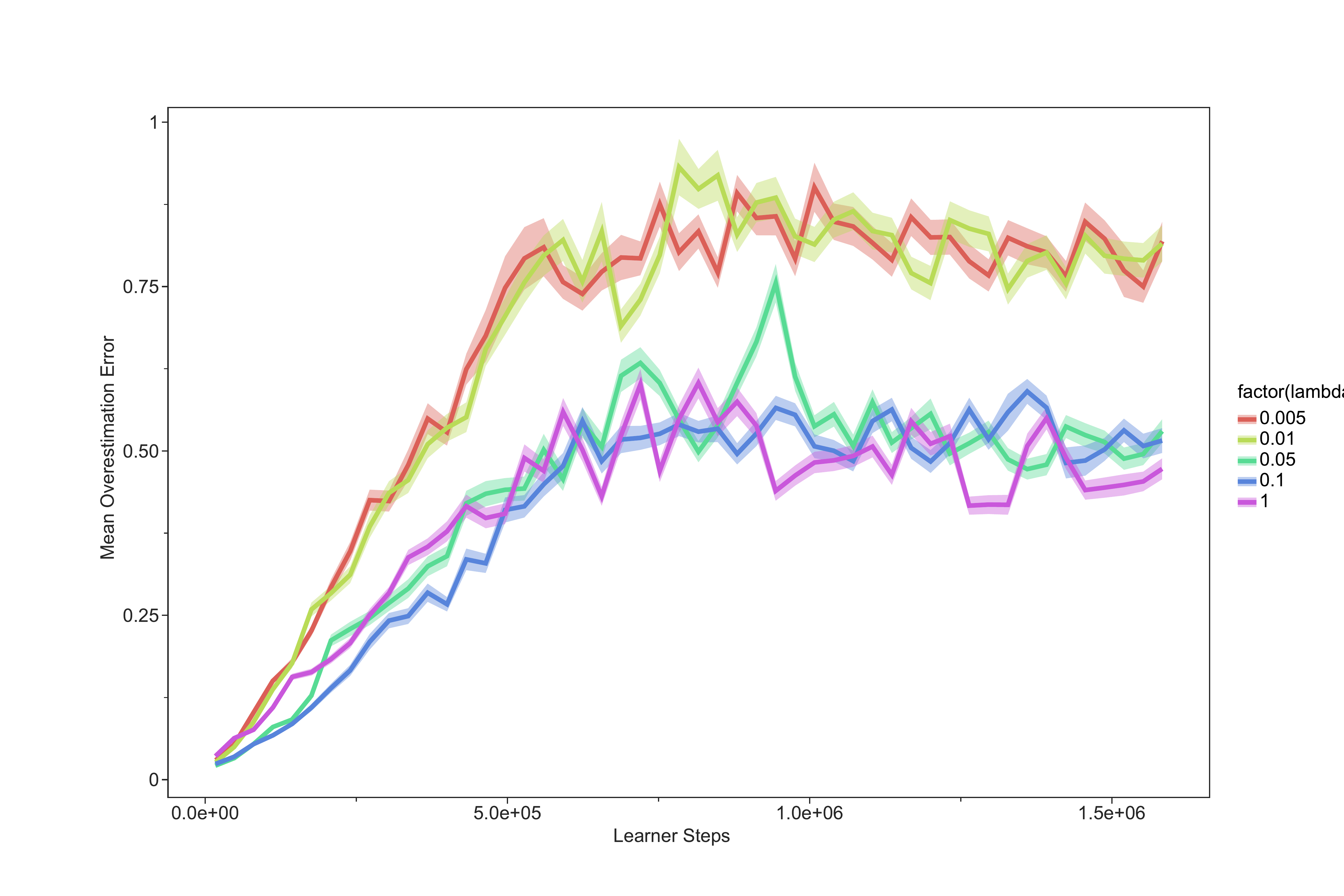}
    \caption{{\bf The Effect of increasing the ranking regularization on the over-estimation.} Here, we see that increasing the ranking regularization reduces the mean over-estimation error across Atari online policy selection games.}
    \label{fig:over-estimation_regularization}
\end{figure}

\subsection{Atari: Robustness to Data}
\label{sec:apx:atari_robustness}
The robustness of the reward distribution in the dataset is an important feature required to deploy offline RL algorithms in the real-world. We would like to understand the robustness of behavior value estimation in the offline RL setting. Thus, we first investigate the robustness of $\mathtt{B}$ in contrast to Q-learning with respect to the datasets' size and the reward distribution. In Figure~\ref{fig:SARSA_vs_DQN_reward_filtering}, we split out the dataset into two smaller datasets: i) transitions coming from only highly rewarding ii) transitions from only poorly performing episodes. We show that $\mathtt{B}$ outperforms Q-learning in both settings.

\begin{figure}[htp!]
    \centering
    \includegraphics[width=\linewidth]{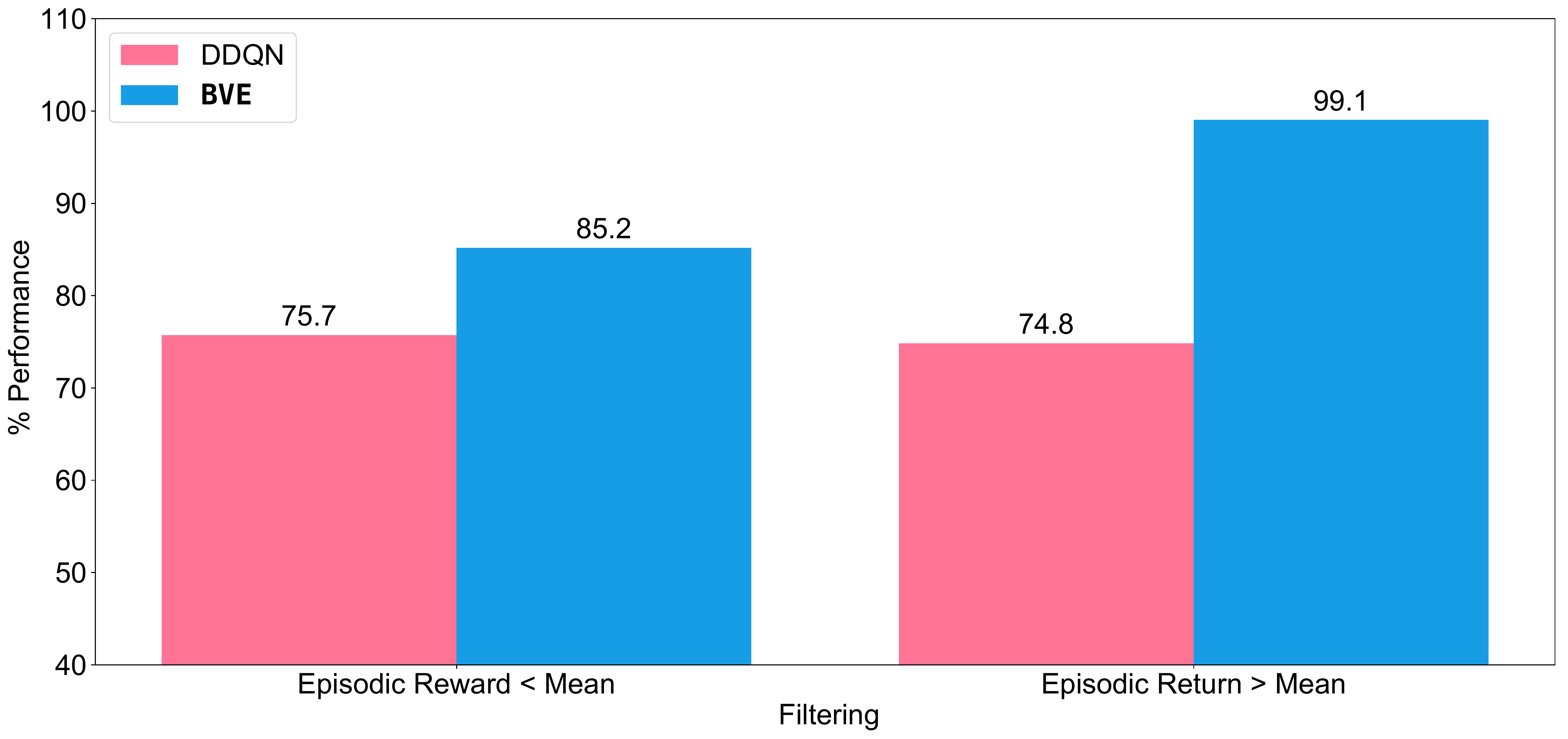}
    \caption{{\bf Robustness Experiments:} We compare DQN and $\mathtt{B}$ in terms of their robustness to the reward distribution on Atari online policy selection games. We split the datasets in two bins: the dataset that only contains transitions that are coming from episodes that have episodic return less than the mean episodic return in the dataset ("Episodic Reward $<$ Mean"), transitions coming from episodes with return higher than the mean return in the dataset ("Episodic Reward $>$ Mean"). $\mathtt{B}$ performs better than DQN in both cases.}
    \label{fig:SARSA_vs_DQN_reward_filtering}
\end{figure}

\subsection{Online Policy Selection Games Results}
In Figure~\ref{fig:atari_ops_returns}, we compare the performance of DDQN, $\mathtt{BVE}$, $\mathtt{R}$-$\mathtt{BVE}$, and $\mathtt{R}$-$\mathtt{DQN}$ with respect to the rewards they achieve over the course of training on Atari online policy selection games.

\begin{figure}[H]
    \centering
    \includegraphics[width=0.85\columnwidth]{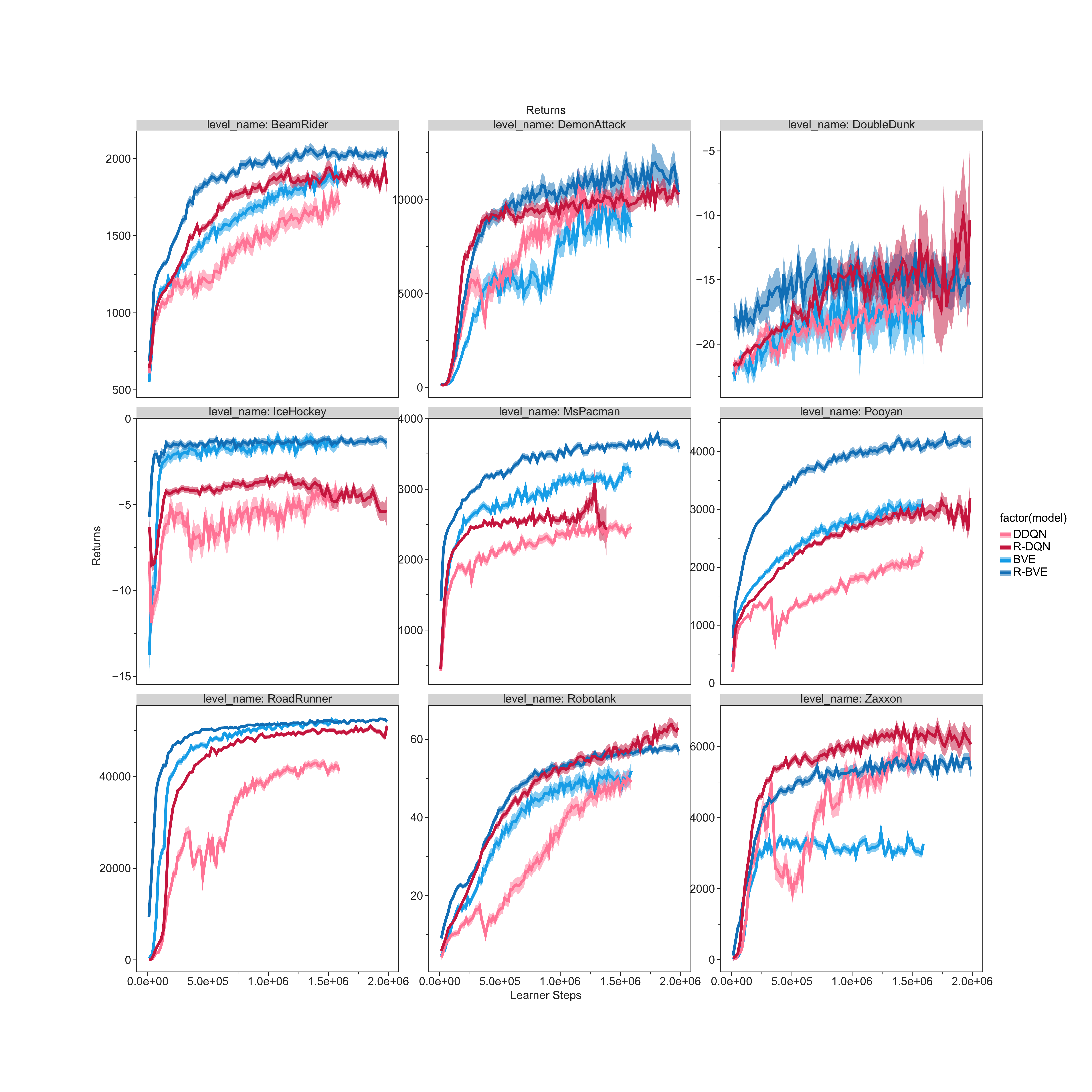}
    \caption{The Raw Returns obtained by each baseline on Atari online Policy Selection Games.}
    \label{fig:atari_ops_returns}
\end{figure}

\subsection{Over-estimation on Online Policy Selection Games}

In Figure \ref{fig:value_err_atari} and \ref{fig:sqr_val_err_atari}, we report the value error of $\mathtt{BVE}$, $\mathtt{R}$-$\mathtt{BVE}$, $\mathtt{R}$-$\mathtt{DQN}$ and DDQN's value error and overestimation error respectively. With respect to both metrics we observed that DDQN has the highest value error. Using $\mathtt{BVE}$ significantly alleviates the problem with the value and overestimation errors, but ranking regularization further reduces the problem both for DDQN and $\mathtt{BVE}$.

\begin{figure}[H]
    \centering
    \includegraphics[width=0.85\columnwidth]{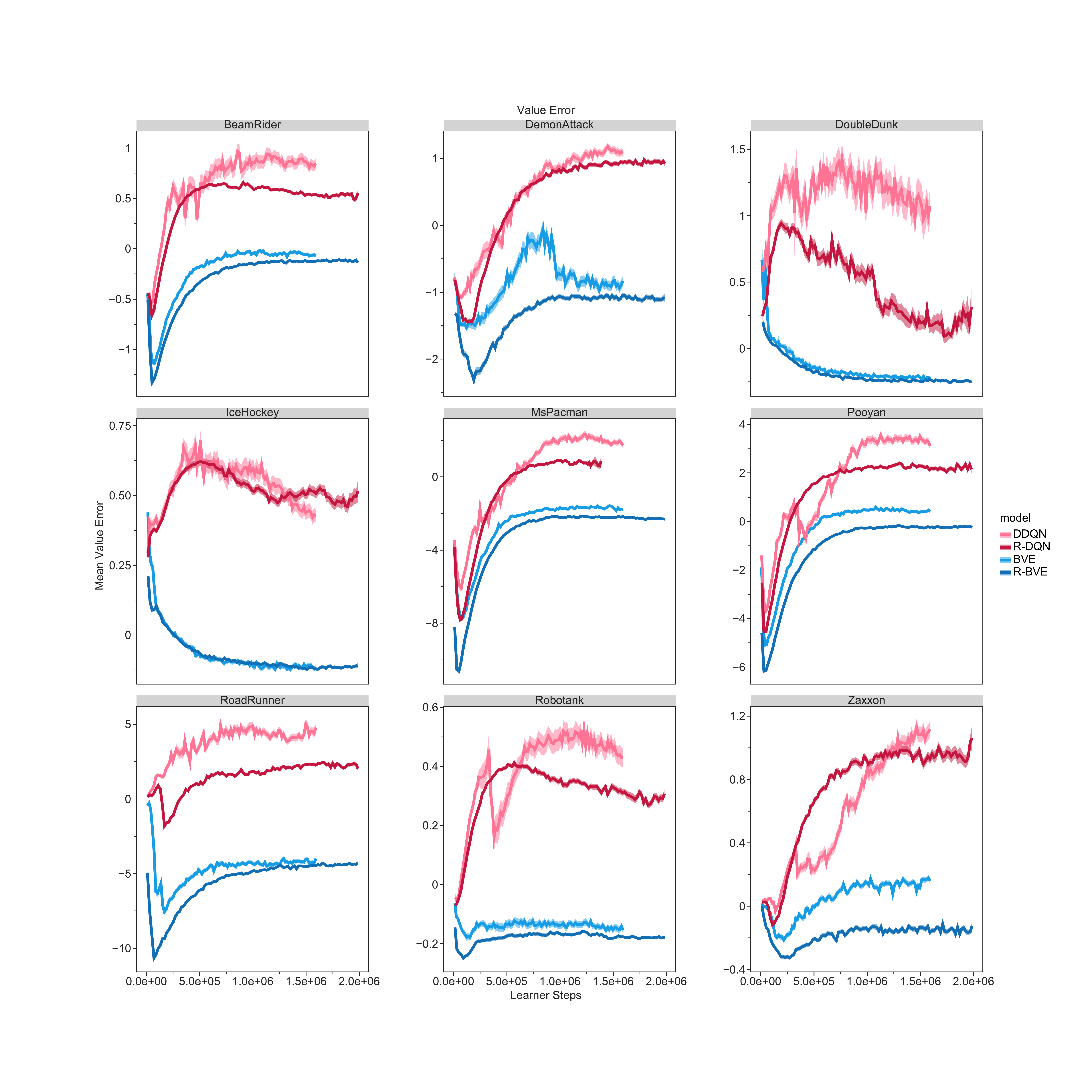}
    \caption{The value error computed in the environment by evaluating the agent and computed with respect to the ground truth discounted returns. The negative values indicate under-estimation and positive values are for over-estimation.}
    \label{fig:value_err_atari}
\end{figure}

\begin{figure}[H]
    \centering
    \includegraphics[width=0.85\columnwidth]{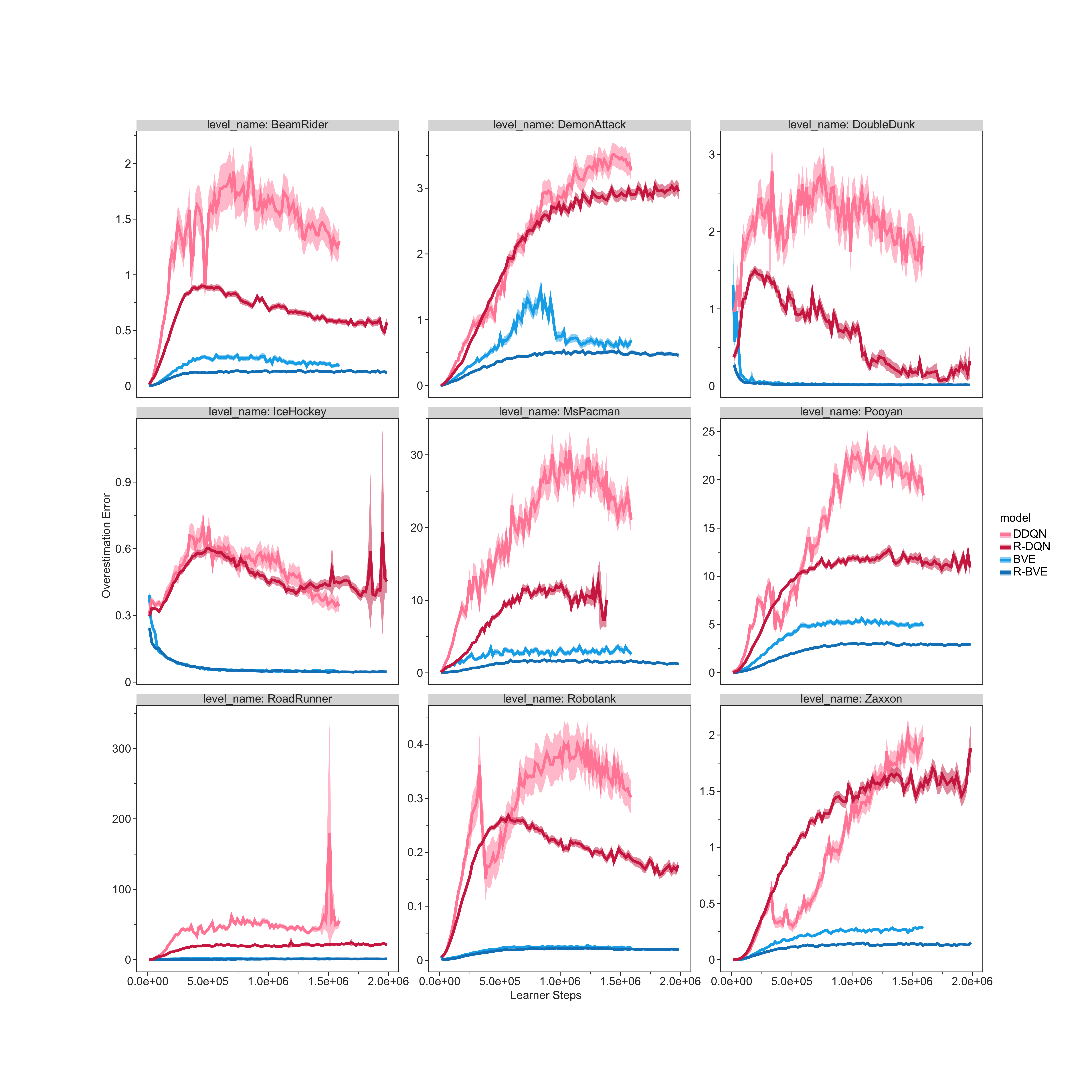}
    \caption{{\bf Overestimation error:} Here we show the over-estimation error computed by evaluating the agent in the environment. We compute the over-estimation error as we discussed in the main paper. Here we can see the clear trend for the over-estimation error across different Atari online policy selection games as follows: DDQN > $\mathtt{R}$-$\mathtt{DQN}$ > $\mathtt{BVE}$ > $\mathtt{R}$-$\mathtt{BVE}$. We see that $\mathtt{BVE}$ alleviates the over-estimation problem and ranking regularization further helps as well.}
    \label{fig:sqr_val_err_atari}
\end{figure}

\section{Details of Datasets}
\label{sec:datasets}

\subsection{BSuite Dataset}
\label{sec:bsuite_dataset}
BSuite \citep{osband2019behaviour} data was collected by training DQN agents \citep{mnih2015humanlevel} with the default setting in Acme \citep{hoffman2020acme} from scratch in each of the three tasks: cartpole, catch, and mountain\_car. We convert the originally deterministic environments into stochastic ones by randomly replacing the agent action with a uniformly sampled action with a probability of $\epsilon \in \{0, 0.1, 0.2, 0.3, 0.4, 0.5\}$ (ie. $\epsilon = 0$ corresponds to the original environment). We train agents (separately for each randomness level and 5 seeds, i.e. 25 agents per game) for 1000, 2000, 500 episodes in cartpole, catch and mountain\_car respectively. The number of episodes is chosen so that agents in all levels can reach their best performance. We record all the experience generated through the training process. Then to reduce the coverage of the datasets and make them more challenging we only used 10\% of the data by subsampling it. More details of the dataset are provided in Table~\ref{tab:bsuite_data}. The results presented in the paper are averaged over the 5 random seeds.

\begin{table}[tbhp!]
    \centering
    \begin{tabular}{c|ccc}
        Environments & Number of episodes & Number of transitions & Average episode length \\
        \hline
        cartpole ($\epsilon=0.0$) & 1000 & 710K & 710 \\
        cartpole ($\epsilon=0.1$) & 1000 & 773K & 773 \\
        cartpole ($\epsilon=0.2$) & 1000 & 649K & 649 \\
        cartpole ($\epsilon=0.3$) & 1000 & 607K & 607 \\
        cartpole ($\epsilon=0.4$) & 1000 & 672K & 672 \\
        cartpole ($\epsilon=0.5$) & 1000 & 643K & 643 \\
        \hline
        catch ($\epsilon=0.0$) & 200 & 1.8K & 9 \\
        catch ($\epsilon=0.1$) & 200 & 1.8K & 9 \\
        catch ($\epsilon=0.2$) & 200 & 1.8K & 9 \\
        catch ($\epsilon=0.3$) & 200 & 1.8K & 9 \\
        catch ($\epsilon=0.4$) & 200 & 1.8K & 9 \\
        catch ($\epsilon=0.5$) & 200 & 1.8K & 9 \\
        \hline
        mountain\_car ($\epsilon=0.0$) & 50 & 10K & 205 \\
        mountain\_car ($\epsilon=0.1$) & 50 & 10K & 210 \\
        mountain\_car ($\epsilon=0.2$) & 50 & 22K & 447 \\
        mountain\_car ($\epsilon=0.3$) & 50 & 13K & 277 \\
        mountain\_car ($\epsilon=0.4$) & 50 & 12K & 250 \\
        mountain\_car ($\epsilon=0.5$) & 50 & 24K & 494 \\
        \hline
    \end{tabular}
    \caption{BSuite dataset details.}
    \label{tab:bsuite_data}
\end{table}

\subsection{DeepMind Lab Dataset}
\label{sec:dmlab_dataset}
DeepMind Lab \citep{beattie2016deepmind} data was collected by training distributed R2D2 \citep{kapturowski2018recurrent} agents from scratch on individual tasks. First, we tuned the hyperparameters of a distributed version of the Acme \citep{hoffman2020acme} R2D2 agent independently for every task to achieve fast learning in terms of actor steps. Then, we recorded the experience across all actors during entire training runs a few times for every task. Training was stopped after there was no further progress in learning across all runs, with a resulting number of steps for each run between 50 million for the easiest task ($\mathtt{seekavoid\_arena\_01}$) and 200 million for some of the hard tasks. Finally we built a separate offline RL dataset for every run and every task. See more details about these datasets in Table~\ref{tab:dmlab_data}.

Additionally, for the $\mathtt{seekavoid\_arena\_01}$ task we ran two fully trained snapshots of our R2D2 agents on the environment with different levels of noise ($\epsilon=0,0.01,0.1,0.25$ for $\epsilon$-greedy action selection). We recorded all interactions with the environment and generated a different offline RL dataset containing 10 million actor steps for every agent and every value of $\epsilon$.

\begin{figure}
    \centering
    \includegraphics[width=\columnwidth]{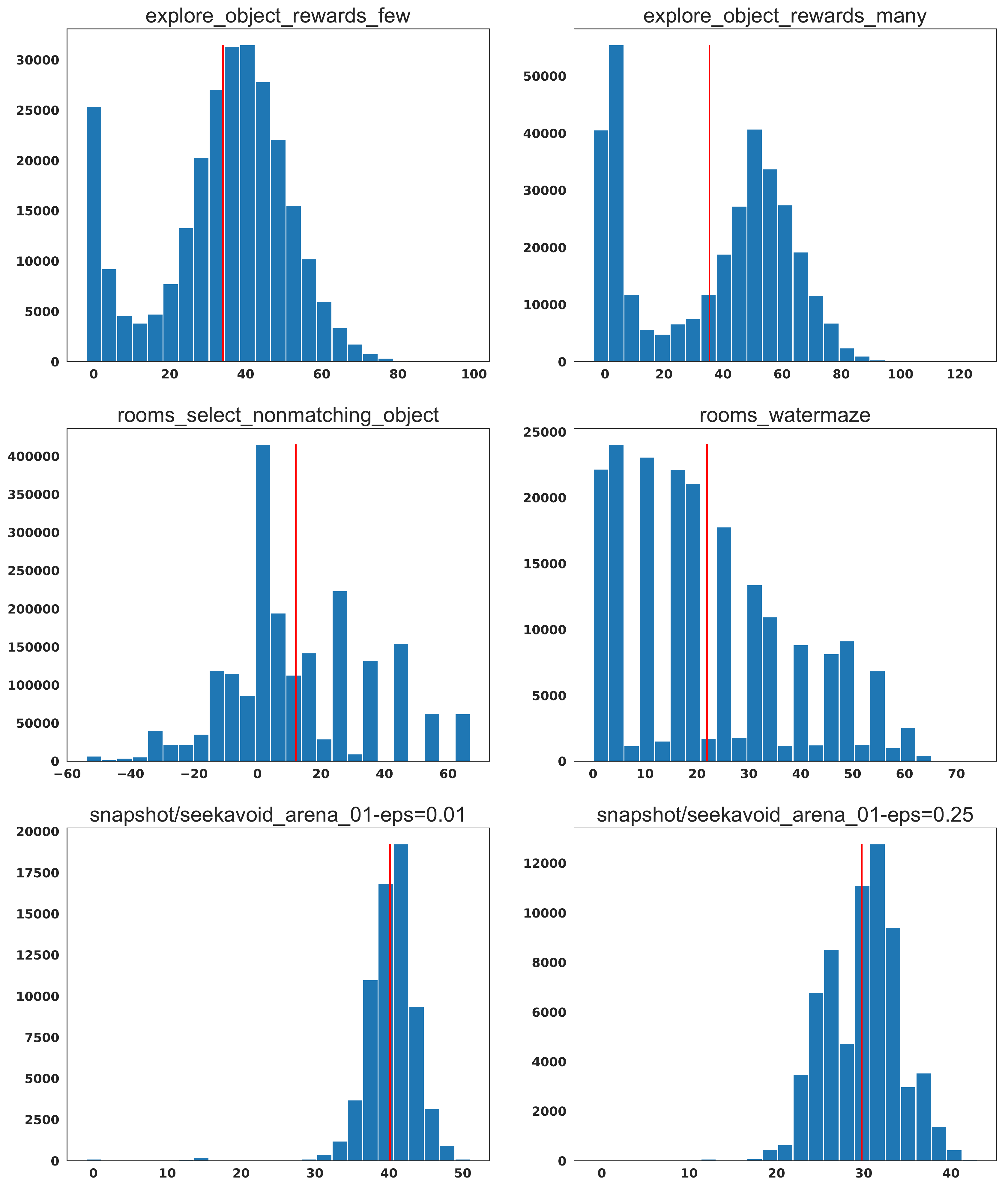}
    \caption{{\bf DeepMind Lab Reward Distribution:} We show the reward distributions for the DeepMind Lab datasets. The vertical red line indicates the average episodic return in the datasets.}
    \label{fig:dmlab_task_rewards_dist}
\end{figure}
\begin{table}[htp!]
\setlength{\tabcolsep}{4pt}
\small	
\centering
\caption{\textbf{DeepMind Lab dataset details.} For training data, reward is measured as the maximum over training of the average reward over runs for the same task. For snapshot data, reward is just an average over all episodes recorded using the same level of noise.}
\label{tab:dmlab_data}
\small
\begin{tabular}{lccccc}
\toprule
Task & Episode Length & Datasets & Episodes (K) & Steps (M) & Reward \\
\midrule
$\mathtt{seekavoid\_arena\_01}$ & 300 & 5 & 667.1 & 200.1 & 39.0 \\
$\mathtt{seekavoid\_arena\_01}$ snapshot ($\epsilon=0$) & 300 & 2 & 66.7 & 20 & 40.4 \\
$\mathtt{seekavoid\_arena\_01}$ snapshot ($\epsilon=0.01$ & 300 & 2 & 66.7 & 20 & 40.1 \\
$\mathtt{seekavoid\_arena\_01}$ snapshot ($\epsilon=0.1$) & 300 & 2 & 66.7 & 20 & 36.9 \\
$\mathtt{seekavoid\_arena\_01}$ snapshot ($\epsilon=0.25$) & 300 & 2 & 66.7 & 20 & 29.7 \\
$\mathtt{explore\_object\_rewards\_few}$ & 1350 & 3 & 178.3 & 240.7 &  51.5 \\
$\mathtt{explore\_object\_rewards\_many}$ & 1800 & 3 & 334.1 & 601.4 & 64.5 \\
$\mathtt{rooms\_select\_nonmatching\_object}$ & 180 & 3 & 2001.1 & 360.2 & 32.5 \\
$\mathtt{rooms\_watermaze}$ & 1800 & 3 & 201.8 & 363.3 & 48.8 \\
\bottomrule
\end{tabular}
\end{table}

\section{Experiment Details}
\label{sec:hyperparam_details}
We used the Adam optimizer \citep{kingma2014adam} for all our experiments. For details on the used hyperparameters, refer to the Table \ref{table:bsuite_hparams} for \texttt{bsuite}, Table \ref{table:atari_hparams} for Atari, and Table \ref{table:deepmind_lab_hparams} for DeepMind Lab. Our evaluation protocol is described below, in Section \ref{sec:evaluation_protocol}. On Atari experiments, we have normalized the agents' scores as described in \citep{gulcehre2020rl}.   On Atari, in all our experiments we report the median normalized score only as accustomed in the literature for reporting results on Atari.

\paragraph{Atari Hyperparameters:} On Atari we directly used the baselines and the hyperparameters reported in \citep{gulcehre2020rl}, to get the detailed Atari results on test set we communicated with the authors. We have run additional CQL and our own models with ranking regularization and reparameterization. For CQL we have finetuned both the learning rate from the grid $\left[8e-5,~1e-4,~3e-4\right]$. For our own proposed models we have only tuned the learning rate from the grid $\left[8e-5,~1e-4,~3e-4\right]$ and the ranking regularization hyperparameter from the grid $\left[0.005, 0.05, 0.01, 0.1, 1\right]$. We have fixed the rest of the hyperparameters. As mentioned earlier, we have only used the online policy selection games for finetuning the hyperparameters. As a result of our grid search, we have used learning rate of $1e-4$ for CQL and our models. We have used $0.01$ for the $\alpha$ hyperparameter of CQL. $0.005$ seems to be the most optimal hyperparameter choice for the ranking regularization hyperparameter on most tasks that we have tried it on.  We provide the details of the Atari hyperparameters and compute infrastructure in Table \ref{table:atari_hparams}.
\begin{table}[ht]
\small
\caption{\textbf{Atari experiments' hyperparameters.} The top section of the table corresponds to the shared hyperparameters of the offline RL methods and the bottom section of the table contrasts the hyperparameters of Online vs Offline DQN.}
\centering
\begin{tabular}{lrr}
\toprule
Hyperparameter & \multicolumn{2}{r}{setting (for both variations)} \\
\midrule
Discount factor && 0.99 \\
Mini-batch size && 256 \\
Target network update period & \multicolumn{2}{r}{every 2500 updates} \\
Evaluation $\epsilon$ && $0.4^8$ \\
$Q$-network: channels && 32, 64, 64 \\
$Q$-network: filter size && $8\times8$, $4\times4$, $3\times3$\\
$Q$-network: stride && 4, 2, 1\\
$Q$-network: hidden units && 512 \\
Training Steps && 2M learning steps \\
Hardware && Tesla V100 GPU \\
Replay Scheme && Uniform \\
\midrule
Hyperparameter & Online & Offline\\
\midrule
Min replay size for sampling & 20,000 & - \\
Training $\epsilon$~(for $\epsilon$-greedy exploration) & 0.01 & - \\
$\epsilon$-decay schedule & 250K steps & - \\
Fixed Replay Memory & No & Yes \\
Replay Memory size & 1M steps & 2M steps \\
Double DQN & No & Yes \\
\bottomrule
\end{tabular}
\label{table:atari_hparams}
\end{table}
\paragraph{bsuite Hyperparameters:} 
Our hyperparameter selection protocol for bsuite is the same as Atari. The main difference between the Atari and bsuite experiments is the network architecture that has been used. We provide the details of hyperparameters used for bsuite in Table \ref{table:bsuite_hparams}.
\begin{table}[ht!]
\small
\caption{\textbf{\texttt{bsuite} experiments' hyperparameters.} The top section of the table corresponds to the shared hyperparameters of the offline RL methods and the bottom section of the table contrasts the hyperparameters of Online vs Offline DQN.}
\centering
\begin{tabular}{lrr}
\toprule
Hyperparameter & \multicolumn{2}{r}{setting (for both variations)} \\
\midrule
Discount factor && 0.99 \\
Mini-batch size && 128 \\
Target network update period & \multicolumn{2}{r}{every 2500 updates} \\
Evaluation $\epsilon$ && $0.4^8$ \\
$Q$-network: & & an MLP \\
$Q$-network: hidden units && $56, 56, \text{num\_actions}$ \\
Training Steps & & 2M learning steps \\
Hardware & & Tesla V100 GPU \\
Replay Scheme & & Uniform \\
\midrule
Hyperparameter & Online & Offline\\
\midrule
Min replay size for sampling & 20,000 & - \\
Training $\epsilon$~(for $\epsilon$-greedy exploration) & 0.01 & - \\
$\epsilon$-decay schedule & 250K steps & - \\
Fixed Replay Memory & No & Yes \\
Replay Memory size & 1M steps & 2M steps \\
Double DQN & No & Yes \\
\bottomrule
\end{tabular}
\label{table:bsuite_hparams}
\end{table}

\paragraph{DeepMind Lab Hyperparameters:} On DeepMind Lab experiments, we tuned the hyperparameters of each model individually on each level separately. We have tuned the learning rate and the method-specific hyperparameters for each model from the same grid that we have used for Atari. For CQL, the specific hyperparameter that we tuned in addition to the learning rate is the regularization hyperparameter $\alpha$. For our own baselines with the ranking regularization, we fixed the ranking regularization co-efficient to the best value we found on Atari, and only tuned the margin hyperparameter $\nu$ from the grid $\{0.05,0.5\}$. All our algorithms use n-step returns in our DeepMind Lab experiments, where $n$ is fixed to $5$ in all our experiments. Thus both behavior value estimation and Q-learning experiments use 5 steps of unrolls for learning. We provide the details of the Deepmind Lab hyperparameters and details of compute infrastructure in Table \ref{table:deepmind_lab_hparams}.

\begin{table}[ht]
\small
\caption{\textbf{Deepmind Lab experiments' hyperparameters.} The top section of the table corresponds to the shared hyperparameters of the offline RL methods and the bottom section of the table contrasts the hyperparameters of Online vs Offline DQN.}
\centering
\begin{tabular}{lrr}
\toprule
Hyperparameter & \multicolumn{2}{r}{setting (for both variations)} \\
\midrule
Discount factor && 0.997 \\
Target network update period & \multicolumn{2}{r}{every 400 updates} \\
Evaluation $\epsilon$ && $0.4^8$ \\
Importance sampling exponent && 0.6 \\
Architecture & \multicolumn{2}{r}{Canonical R2D2 \citep{kapturowski2018recurrent}} \\
\midrule
Hyperparameter & Online & Offline\\
\midrule
Hardware & 4x TPUv2 & 4x Tesla V100 GPU \\
Training Steps & 50-200M actor steps & 50K learning steps \\
Sequence Length & 120 (40 burn-in) & Full episode \\
Mini-batch size & 32 & 8 \\
Training $\epsilon$~(for $\epsilon$-greedy exploration) & $0.4, ..., 0.4^8$ & - \\
Replay Scheme & Prioritized (exponent 0.9) & - \\
Min replay size for sampling & 600K steps & - \\
Replay Memory size & 12M steps & 50-200M steps \\
\bottomrule
\end{tabular}
\label{table:deepmind_lab_hparams}
\end{table}

\subsection{Evaluation protocol}
\label{sec:evaluation_protocol}
To evaluate the performance of the various methods, we use the following protocol:
\begin{enumerate}
    \item We sweep over a small (5-10) sets of hyperparameter values for each of the methods.
    \item We independently train each of the models on 5 datasets generated by running the behavior policy with 5 different seeds (ie. producing 25-50 runs per problem setting and method).
    \item We evaluate the produced models in the original environments (without the noise).
    \item We average the results over seeds and report the results of the best hyperparameter for each method.
\end{enumerate}

\subsubsection{Evaluation method}
To evaluate models (step 3. above), in the case of \texttt{bsuite} and DeepMind Lab we ran an evaluation job in parallel to the training one. It repeatedly read the learner's checkpoint and produced evaluation results during training. We report the average of the evaluation scores over the last 100 learning steps.

In the case of the Atari environments, instead of averaging performance during the final steps of learning, we take the final snapshot produced by a given method and evaluate it on a `100` environment steps after the training finished.

\subsection{Atari Offline Policy Selection Results}
In Table \ref{tab:atari_testing_results}, we show the performance of our baselines on different Atari Offline Policy selection games. We show that $\mathtt{R}$-$\mathtt{BVE}$ outperforms other approaches significantly.
\begin{table}[htbp!]
\centering
\caption{{\bf Atari Offline Policy Selection Results}: In this table, we list the median normalized performance of different baselines across all RLU Atari offline policy selection games.}
\label{tab:atari_testing_results}
\begin{tabular}{@{}l|l@{}}
\toprule
Name & Normalized Score \\ \midrule
BC   & 50.8             \\ \midrule
DDQN & 83.1             \\ \midrule
CQL  & 98.9             \\ \midrule
BCQ  & 102.6            \\ \midrule
IQN  & 104.8            \\ \midrule
REM  & 104.7            \\ \midrule
$\mathtt{R}$-$\mathtt{DQN}$  & { 108.2} 
\\ \midrule
$\mathtt{R}$-$\mathtt{BVE}$  & {\bf 109.1} 
\\ \bottomrule
\end{tabular}
\end{table}

\subsection{DeepMind Lab Detailed Results}

In Table \ref{tab:dmlab_results}, we have shown the results on the Deepmind Lab datasets. It is possible see from these numerica results that $\mathtt{R}$-$\mathtt{BVE}$ outperforms other approaches and $\mathtt{BVE}$ is still very competitive.
\begin{table}[ht]
\caption{{\bf Detailed Results on the DeepMind Lab:} We provide the detailed results for each DeepMind levels along with the standard errors.}
\label{tab:dmlab_results}
\small
\begin{tabular}{@{}l|l|l|l|l|l@{}}
                                   & BC & R2D2 & CQL & $\mathtt{BVE}$ & $\mathtt{R}$-$\mathtt{BVE}$ \\ \midrule
$\mathtt{explore\_object\_rewards\_many}$ & 1.8  $\pm$  1.0   &   19.8  $\pm$ 4.0    & 23.8  $\pm$ 5.1   &  23.7  $\pm$ 3.8 & {\bf 31.4}  $\pm$ 1.7   \\ \midrule
$\mathtt{explore\_object\_rewards\_few}$  & 2.9  $\pm$ 1.4   & 8.5  $\pm$   3.4  & 9.3   $\pm$ 2.5  & 7.6  $\pm$ 2.1 & {\bf 13.4}  $\pm$  2.6 \\ \midrule
$\mathtt{rooms\_watermaze}$ & 0.1   $\pm$  0.1 & 2.7    $\pm$  1.4 & 4.0  $\pm$  3.7  & 9.9   $\pm$ 2.7 & {\bf 14.1  $\pm$ 4.2  }\\ \midrule
$\mathtt{rooms\_select\_nonmatching\_object}$ & 1.1  $\pm$ 4.6 & 5.4   $\pm$ 2.3   &  3.4  $\pm$ 2.4 & 9.4  $\pm$ 2.3   & {\bf 9.6}  $\pm$ 3.2 \\ \midrule
$\mathtt{seekavoid\_arena\_01}$, $\epsilon=0$    & {\bf 28.02} $\pm$ 7.6  &   4.7 $\pm$  3.0  & 12.8 $\pm$ 10.7  &  4.4 $\pm$ 0.9   & 17.07 $\pm$ 10.1   \\ \midrule
$\mathtt{seekavoid\_arena\_01}$, $\epsilon=0.01$ & {\bf 32.4} $\pm$ 1.3   &    5.5 $\pm$ 1.6    & 12.7 $\pm$ 5.4     & 4.1 $\pm$ 1.8  & 19.8 $\pm$ 4.9   \\ \midrule
$\mathtt{seekavoid\_arena\_01}$, $\epsilon=0.1$  & 18.9 $\pm$ 14.4 &   8.6 $\pm$ 3.0   & 16.3 $\pm$ 7.7    & 11.775 $\pm$ 4.5  & {\bf 31.8} $\pm$ 4.7  \\ \midrule
$\mathtt{seekavoid\_arena\_01}$, $\epsilon=0.25$ & 17.46 $\pm$ 7.5    &  13.5 $\pm$ 5.1    & 13.5 $\pm$ 5.06    & 9.0 $\pm$ 0.25   &  {\bf 25.57} $\pm$ 7.0 
\end{tabular}
\end{table}

\section{Ranking Regularization}

We propose a family of methods that prevent the extrapolation error by suppressing the values of the actions that are not in the dataset. We achieve that by ranking the actions in the training set higher than the ones that are not in the training set. For the learned Q-function the absolute values of actions do not matter, we are rather interested in relative ranking of the actions.
Given  $a_t$ is the action from the dataset. For all  $j \ne t$ and illustration purposes, the value iteration can be written as:
\begin{small}
\begin{align*}
\E[\max_{a}Q(s, a)] &\approx \E\left[P(Q(s, a_t) > Q(s, a_j)) Q(s, a_t)|t \in Max\right] + \E\left[P(Q(s, a_t) \ngtr Q(s, a_j)) Q(s, a_j)|j \in Max\right] \\
&= \E\left[P(Q(s, a_t) > Q(s, a_j)) Q(s, a_t)|t \in Max\right] + \E\left[(1-P(Q(s, a_t) > Q(s, a_j))) Q(s, a_j)|j \in Max\right] \\
&= \alpha\E\left[P(\Qhat(s, a_t) > \Qhat(s, a_j)) \Qhat(s, a_t)|t \in Max\right] + \alpha\E\left[(1-P(\Qhat(s, a_t) > \Qhat(s, a_j))) \Qhat(s, a_j)|j \in Max\right] \\
& = \alpha \left(\E\left[P(\Qhat(s, a_t) > \Qhat(s, a_j))\Qhat(s, a_t)|t \in Max\right] + \E\left[(1-P(\Qhat(s, a_t) > \Qhat(s, a_j)))\right]\right)\xi \\
\end{align*}
\end{small}
where  $\xi$ is an irreducible noise, because we can not gather additional data on $(s_t,~a_j)$, and we don't know the corresponding reward for it. This causes extrapolation error which accumulates through the bootstrapping in the backups as noted by \cite{kumar2019stabilizing}. We implicitly pull down the $P(Q(s, a_t) \ngtr Q(s, a_t))$ by ranking the actions in the dataset higher which pushes up $ P(Q(s, a_t) > Q(s, a_j))$. As a result, the extrapolation error in Q-learning would also reduce.

\subsection{Pairwise Ranking Loss for Q-learning}
\label{sec:pairwise_ranking_loss}
In this section, we discus the relationship between the pairwise ranking loss for Q-learning and the list-wise pairwise ranking losses.
\begin{align*}
    p_{tj} &= \sigmoid(\hat{Q}_{\theta}(s, a_t) - \hat{Q}_{\theta}(s, a_j)), \\
    \pi(a_t|s) &\approx \prod_{i=0,i \neq t}^{|\calA|}p_{ti} / Z,\\
    Z &=  \sum_{i=0}^{|\calA|} \prod_{j=0,j \neq i}^{|\calA|}p_{ij}.\\
    \calR(\theta) &= - \sum_{i=0}^{|\calA|}\log(p_{ti}), \\
    &= - \sum_{i=0,i \neq t}^{|\calA|}\log(\sigmoid(\hat{Q}_{\theta}(s, a_t) - \hat{Q}_{\theta}(s, a_i))), \\
     &= \sum_{i=0,i \neq t}^{|\calA|}\text{softplus}(\hat{Q}_{\theta}(s, a_i) - \hat{Q}_{\theta}(s, a_t)). \\
\end{align*}
We use a common approximation \citep{chen2009ranking,burges2005learning} of the $\text{softplus}$ with the $\max(\cdot, 0)$ function:
\begin{equation}
    \mathcal{C}(\theta) = \sum_{i=0,i \neq t}^{|\calA|} \max\left(\hat{Q}_{\theta}(s,a_i) - \hat{Q}_{\theta}(s,a_t) + \nu, 0\right)^2
    \label{eqn:quirk_reg}
\end{equation}
Imposing the constraint in Equation \eqref{eqn:quirk_reg} can be harmful when the dataset has large amount of suboptimal trajectories (or when the dataset is generated by a random policy.) Because this constraint can drive the value function to overestimate the values of suboptimal actions in the dataset. As a result, similar to \cite{wang2020critic}, we propose a filtering function to impose that constraints only on rewarding transitions:
\begin{equation}
    \mathcal{C}(\theta) =\exp(\hat{V}^{\pi_{\calB}}(s) - \E_{s\sim \data}[\hat{V}^{\pi_{\calB}}(s)]) \sum_{i=0,i \neq t}^{|\calA|}  \max\left(\hat{Q}_{\theta}(s,a_i) - \hat{Q}_{\theta}(s,a_t) + \nu, 0\right)^2
\end{equation}

\subsection{Relationship to the Policy Gradients}
It is possible to drive the formulation that we use for the ranking regularizer from the policy gradient theorem to show the relationship. The Ranking Policy Gradient Theorem formulates the optimization of long-term reward using a ranking objective as done in \cite{DBLP:conf/iclr/LinZ20}. The equations below illustrates the formulation process. Let us note that we apply the ranking regularization on the offline and off-policy data, such that the formalism below only works when the behavior policy and target policy are equivalent, when the transitions are coming from on-policy data. If the ranking regularizer is used on the on-policy data it approximates the policy gradients, but it will not on the off-policy data.

Our construction is based on direct policy differentiation~\citep{peters2008reinforcement,williams1992simple} where the objective function is to $\theta^\ast = \argmax_{\theta} J({\theta}) $ for $p_{ij} = \text{sigmoid}\left(Q(s, a_i) - Q(s, a_j)\right)$.  
\begin{align}
 \grad_{\theta} J({\theta}) =& \grad_{\theta} \sum_{\tau}\nolimits p_{\theta}(\tau) \hat{V}^{\pi_{\calB}}(s) \label{eq:rlpgobj} \\
=& \sum_\tau\nolimits p_{\theta}(\tau) \grad_{\theta}\log p_{\theta}(\tau)\hat{V}^{\pi_{\calB}}(s) \nonumber \\
=& \sum_\tau\nolimits p_{\theta}(\tau) \grad_{\theta}\log \left(p(s_0)\Pi_{t=1}^{T}\pi_{\theta}(a_t|s_t)p(s_{t+1}|s_t, a_t)\right)\hat{V}^{\pi_{\calB}}(s) \nonumber \\
=& \sum_\tau\nolimits p_{\theta}(\tau)  \sum_{t=1}^{T}\nolimits\grad_{\theta}\log\pi_{\theta}(a_t|s_t)\hat{V}^{\pi_{\calB}}(s) \nonumber \\
=& \E_{\tau \sim \pi_\theta}\left[\sum_{t=1}^{T}\nolimits \grad_{\theta} \log\pi_{\theta}(a_t|s_t) \hat{V}^{\pi_{\calB}}(s)\right ]\nonumber\\
=& \E_{\tau \sim \pi_\theta}\left[\sum_{t=1}^{T}\nolimits\grad_{\theta} \log\left(\prod\nolimits_{j=1, j\neq i}^{m} p_{ij}\right) \hat{V}^{\pi_{\calB}}(s)\right]\nonumber \\
=& \E_{\tau \sim \pi_\theta}\left[\sum_{t=1}^{T}\nolimits\grad_{\theta} \sum_{j=1, j\neq i}^{m}\nolimits \log\left(\text{sigmoid}(p_{ij}) \right)\hat{V}^{\pi_{\calB}}(s)\right]\label{eq:rpg_ori} \\
=& -\E_{\tau \sim \pi_\theta}\left[\sum_{t=1}^{T}\nolimits\grad_{\theta} \sum_{j=1, j\neq i}^{m}\nolimits \text{softplus}(p_{ji})\hat{V}^{\pi_{\calB}}(s)\right]\nonumber\\
\approx& -\E_{\tau \sim \pi_\theta}\left[\sum_{t=1}^{T}\nolimits\grad_{\theta} \left(\sum_{j=1, j\neq i}^{m}\nolimits \max\left(Q(s, a_i)-Q(s, a_j), 0\right)\right)\hat{V}^{\pi_{\calB}}(s)\right],
\label{eq:rpg_first}
\end{align}
with baseline $\E_{s\sim \data}[\hat{V}^{\pi_{\calB}}(s)]$ it will be,
\begin{align}
   \approx& -\E_{\tau \sim \pi_\theta}\left[\sum_{t=1}^{T}\nolimits\grad_{\theta} \left(\sum_{j=1, j\neq i}^{m}\nolimits \max\left(Q(s, a_i)-Q(s, a_j,0)\right)\right)\left(\hat{V}^{\pi_{\calB}}(s)-\E_{s\sim \data}[\hat{V}^{\pi_{\calB}}(s)]\right)\right] \label{eq:rpg_loss}
\end{align}
Then we apply the $\exp(\cdot)$ transformation on $(\hat{V}^{\pi_{\calB}}(s)-\E_{s\sim \data}[\hat{V}^{\pi_{\calB}}(s)]$ to impose this loss loss mostly on the rewarding trajectories, and we can turn the maximization problem to a minimization one with a flip of sign:
\begin{align}
   =& \E_{\tau \sim \pi_\theta}\left[\sum_{t=1}^{T}\nolimits\grad_{\theta} \left(\sum_{j=1, j\neq i}^{m}\nolimits \max\left(Q(s, a_i)-Q(s, a_j), 0\right)\right)\exp\left(\hat{V}^{\pi_{\calB}}(s)-\E_{s\sim \data}[\hat{V}^{\pi_{\calB}}(s)]\right)\right] \label{eq:quirk_loss}
\end{align}
where the trajectory is a series of state-action pairs from $t=1,...,T$, i.e. $\tau = \left(s_1,a_1,s_2,a_2,...,s_T\right)$. The gradients in \eqref{eq:quirk_loss} would exactly be the gradients of the ranking regularization.

\end{document}